%% file: main.tex
\documentclass{article} % For LaTeX2e
\usepackage{iclr2027_conference,times}

\input{math_commands.tex}

\usepackage[utf8]{inputenc}
\usepackage[T1]{fontenc}
\usepackage{hyperref}
\usepackage{url}
\usepackage{booktabs}       % professional-quality tables
\usepackage{amsfonts}       % blackboard math symbols
\usepackage{amsmath}
\usepackage{amssymb}
\usepackage{amsthm}
\newtheorem{proposition}{Proposition}
\theoremstyle{remark}
\newtheorem{remark}{Remark}
\usepackage{nicefrac}       % compact symbols for 1/2, etc.
\usepackage{microtype}      % microtypography
\usepackage{subcaption}
\usepackage{multirow}
\usepackage{xcolor}
\usepackage{colortbl}
\usepackage{enumitem}
\usepackage{graphicx}
\usepackage{wrapfig}
\usepackage{needspace}
\usepackage[most]{tcolorbox}
\definecolor{pacgpromptblue}{HTML}{24639B}
\newtcolorbox[auto counter]{pacgpromptbox}[2][]{
    enhanced, breakable, width=\linewidth,
    colback=pacgpromptblue!6!white,
    colframe=pacgpromptblue!80!black,
    colbacktitle=pacgpromptblue!85!black, coltitle=white,
    fonttitle=\bfseries\small, fontupper=\small,
    title={Prompt~\thetcbcounter: #2},
    attach boxed title to top left={yshift=-0.1in,xshift=0.15in},
    boxed title style={boxrule=0pt,colframe=pacgpromptblue},
    top=8pt,bottom=6pt,left=7pt,right=7pt,#1
}
\usepackage{algorithm}
\usepackage{algpseudocode}
\usepackage{placeins}
\usepackage{flafter}        % keep floats after their first textual reference

\title{Visual Sensitivity Is Not Claim Retractability:\\Persistence-Aware Credit Assignment for Multimodal Reinforcement Learning}

\author{%
Zhongan Bi\textsuperscript{1,*}, Kepeng Lin\textsuperscript{2,*}, Xuanang Gao\textsuperscript{3,*},\\[3pt]
Yuhan Sun, Guoyi Li, Lianrun Zhang\\[8pt]
{\small\textsuperscript{1}Zhejiang University}\\
{\small\textsuperscript{2}Huazhong University of Science and Technology}\\
{\small\textsuperscript{3}Shanghai Jiao Tong University}\\[5pt]
{\small\textsuperscript{*}Equal contribution.}
}
\hypersetup{pdfauthor={Zhongan Bi, Kepeng Lin, Xuanang Gao, Yuhan Sun, Guoyi Li, Lianrun Zhang}}

\iclrfinalcopy
\makeatletter
\renewcommand{\@maketitle}{%
  \vbox{\hsize\textwidth
    \centering
    {\LARGE\bfseries\@title\par}
    \vskip 14pt
    {\normalsize\normalfont\@author\par}
    \vskip 18pt
  }%
}
\makeatother

\newcommand{\pacg}{PACG}

\definecolor{oursblue}{RGB}{230,241,250}
\definecolor{oursbluestrong}{RGB}{210,231,247}
\definecolor{sectiongray}{RGB}{241,243,245}

\begin{document}

\maketitle
\addtocontents{toc}{\protect\setcounter{tocdepth}{-1}}

\input{section/0_abstract}

\begin{figure}[!t]
    \centering
    \includegraphics[width=\linewidth,height=0.43\textheight,keepaspectratio]{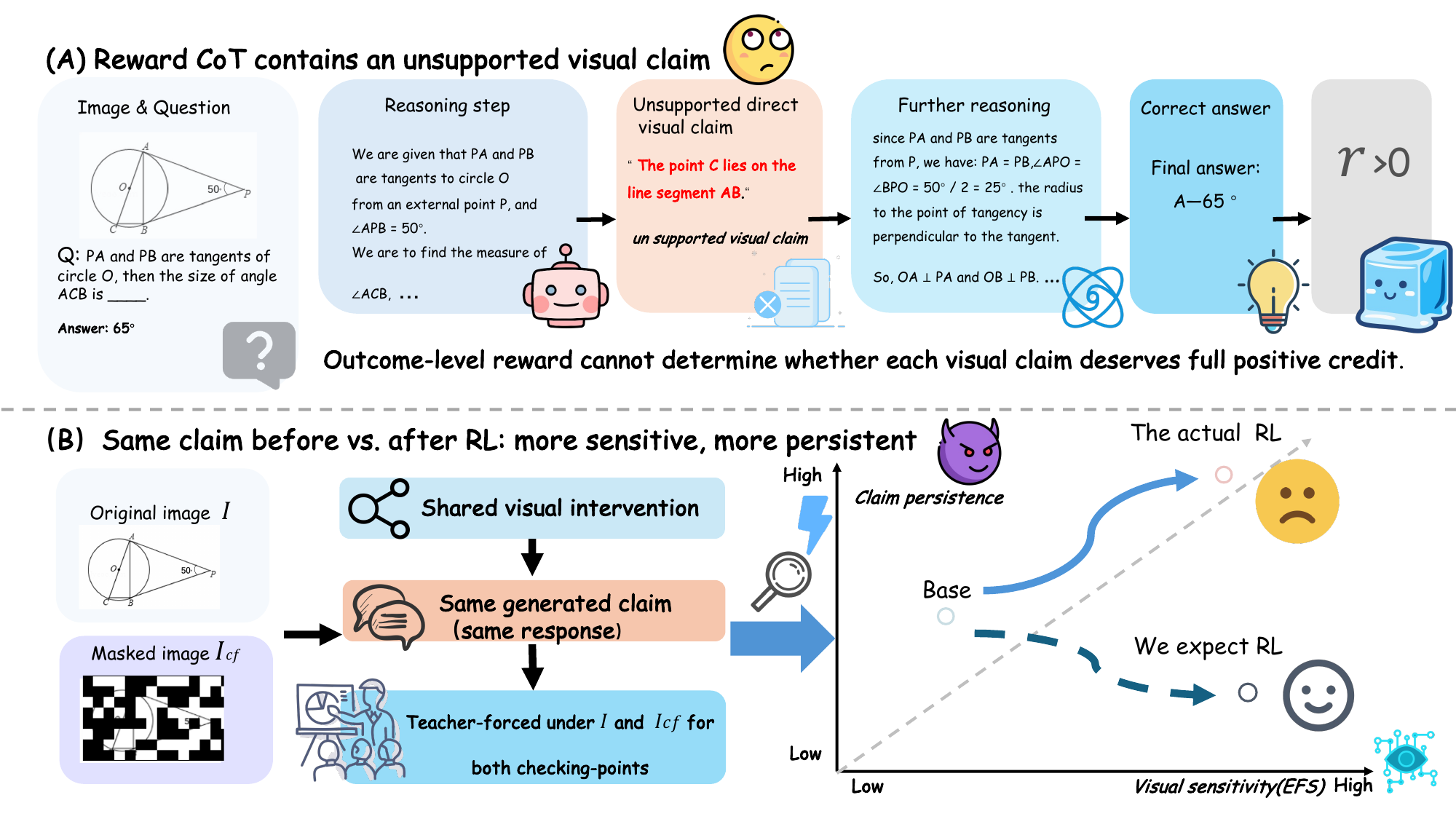}
   \caption{A correct pre-RL answer can conceal an unsupported visual assertion.
We fix this rewarded trace and compare its teacher-forced behavior across
checkpoints, asking whether outcome RL makes the claim more sensitive to image
intervention without making it more retractable.}
    \label{fig:motivating}
\end{figure}

\input{section/1_introduction}
\input{section/2_related_work}
% Queue the method overview early so it can appear at the top of page 4.
\input{figure/pacg_overview}
\input{section/3_preliminaries}
\input{section/4_method}
\input{section/5_experiments}
\FloatBarrier
\input{section/8_conclusion}
\clearpage

% ----------------------------------------------------------------------------
% Bibliography
% ----------------------------------------------------------------------------
\bibliographystyle{iclr2027_conference}
\bibliography{refs}

% ----------------------------------------------------------------------------
% Supplementary material
% ----------------------------------------------------------------------------
\clearpage
\appendix
\addtocontents{toc}{\protect\setcounter{tocdepth}{2}}
\input{section/7_appendix}

\end{document}

%% file: math_commands.tex
\usepackage{amsmath,amsfonts,bm}

\def\eqref#1{equation~\ref{#1}}
\def\1{\bm{1}}

\DeclareMathAlphabet{\mathsfit}{\encodingdefault}{\sfdefault}{m}{sl}
\SetMathAlphabet{\mathsfit}{bold}{\encodingdefault}{\sfdefault}{bx}{n}

%% file: section/0_abstract.tex
\begin{abstract}
Reinforcement Learning with Verifiable Rewards (RLVR) has been extended to
Large Vision--Language Models (LVLMs), and perception-aware methods further
encourage policies to rely on visual evidence. Yet relying on the image does not
guarantee that visual claims are supported by it. Before RL training, 27.81\% of
the correctly answered responses of \mbox{Qwen2.5-VL-7B} on four multimodal
reasoning benchmarks contain at least one direct visual claim that the image
does not support. Since outcome-level RL rewards each response as a whole, these
claims inherit the positive credit of the correct answer. We introduce a
fixed-rollout counterfactual diagnostic that re-scores the same response under
an intervened image to separate Evidence-Function Sensitivity (EFS), how
strongly the model's predictions change, from claim persistence, whether the
model keeps supporting the same claim rather than retracting it. The diagnostic
reveals Sensitivity--Persistence Decoupling (SPD): under DAPO and VPPO, EFS
increases and claims become more retractable overall, yet unsupported claims
become significantly more persistent, whereas GRPO raises EFS without this
deterioration. We therefore propose Persistence-Aware Credit Gating (PACG),
which attenuates positive credit for unusually persistent visual claims and
leaves all other credit unchanged. It requires no supported/unsupported labels
and adds no inference cost. On \mbox{Qwen2.5-VL-7B}, PACG raises the
nine-benchmark average over three seeds from 58.1\% to 59.9\% with DAPO and
from 59.8\% to 60.9\% with VPPO, while making unsupported claims more
retractable. The gains extend to a larger model, a newer backbone, and the accuracy of HallusionBench also improves consistently.
These results suggest that visual sensitivity and claim
retractability are complementary dimensions of multimodal credit assignment.
\end{abstract}

%% file: section/1_introduction.tex
\section{Introduction}
\label{sec:introduction}

Reinforcement Learning with Verifiable Rewards (RLVR) has driven substantial
progress in language model post-training
\citep{shao2024deepseekmath,guo2025deepseek,yu2026dapo,jaech2024openai}
and has been extended to Large Vision--Language Models (LVLMs) for visual
mathematics, chart understanding, and geometric reasoning
\citep{huang2025visionr1,liu2025visual,meng2025mm,zhang2025r1}.
Yet stronger reasoning performance does not imply reliable visual
grounding. LVLMs can retain systematic visual blind spots
\citep{tong2024eyeswideshut}, extended reasoning can amplify visual
hallucination \citep{liu2025morethinking}, and RL gains may persist even under
hallucination-inductive visual conditions \citep{zhang2026hallucinationrl}.
A single trajectory-level reward cannot identify which parts of 
a multimodal reasoning trace deserve reinforcement. 
To address this credit assignment problem, recent work 
redirects supervision toward visual evidence via perception rewards, 
implicit objectives, token-level reweighting, and visually grounded selection \citep{xiao2026perception,wang2025papo,huang2026spotlight,li2026prpo,lu2026bridging,huang2026sketchvl,jiao2026credit}. 
Multimodal process reward models further provide finer-grained supervision
\citep{wang2026visualprm400k,zhou2026whether}. These approaches mainly
determine where and how strongly a policy relies on the image, 
leaving open whether individual visual claims are actually supported by the visual evidence.

We study this question at the level of direct visual claims, the minimal
response spans that assert information read directly from the image, such as
counts, OCR values, or geometric properties. As shown in
Figure~\ref{fig:motivating}(A), such a claim can appear inside a trajectory
that still reaches the correct answer. We call a claim retractable if its
support decreases when the visual evidence is disrupted, and persistent if it
retains that support. The distinction matters because outcome-level RL rewards
trajectories as a whole. We evaluate Qwen2.5-VL-7B on four multimodal reasoning benchmarks and collect the trajectories.
We find that 2,370 of 8,521 answer-correct Base trajectories
(27.81\%) contain at least one unsupported direct visual claim, and
each such claim inherits positive trajectory-level credit despite lacking
visual support. One might expect that making a policy more visually sensitive would also make
these claims easier to retract, but we find that it does not:
\begin{tcolorbox}[
colback=gray!10,
colframe=black,
boxrule=0.8pt,
arc=1.5pt,
top=2pt,
bottom=2pt,
left=3pt,
right=3pt
]
\emph{Greater visual sensitivity does not imply that unsupported visual
claims become more retractable.}
\end{tcolorbox}

To address this, we introduce a fixed-rollout counterfactual diagnostic.
For the same generated claim, we teacher-force the fixed response under the
original image $I$ and an intervened image $I^{\mathrm{cf}}$.
\textbf{Evidence-Function Sensitivity (EFS)} measures how strongly the
predictive distribution at the claim tokens changes under the intervention,
whereas \textbf{claim persistence} measures whether support for the same claim
is retained after its visual evidence is disrupted. The latter is calibrated
against residual tokens from the same rollout to remove trajectory-wide
likelihood drift. Tracking both quantities across checkpoints reveals
\textbf{Sensitivity--Persistence Decoupling (SPD)}. Under DAPO and VPPO, EFS
increases and the overall retractability of visual claims improves, yet
unsupported claims become more persistent. GRPO increases EFS without the same deterioration,
which shows that the relation between sensitivity and retractability depends on the optimizer rather than being universal.

Building on this diagnosis, we propose
\textbf{Persistence-Aware Credit Gating (\pacg{})}, which uses claim
persistence to gate positive token credit. Routed direct visual claims that
retract by a sufficient margin keep their full positive advantage, whereas
unusually persistent claims receive attenuated positive credit. Negative
advantages and non-routed tokens are left unchanged. \pacg{} uses no
supported/unsupported labels during training and adds no inference-time cost.
Across nine multimodal reasoning benchmarks, \pacg{} improves DAPO on all
nine benchmarks and further improves VPPO, achieving the highest aggregate
among our matched 7B runs, with both gains significant across three training
seeds. With both parent optimizers, it significantly reduces
unsupported-claim persistence while preserving or improving EFS and
final-answer accuracy, including on HallusionBench.

Our contributions are threefold:
\begin{itemize}[leftmargin=*, nosep]
\item We identify a credit-assignment failure in multimodal RL: unsupported
direct visual claims frequently occur inside answer-correct trajectories and
can inherit positive trajectory-level credit.

\item We introduce a fixed-rollout counterfactual diagnostic that separates
visual sensitivity from claim persistence and uncover
Sensitivity--Persistence Decoupling: multimodal RL can
increase visual sensitivity and aggregate retractability while making
unsupported visual claims more persistent.

\item We propose \pacg{}, a conservative and composable credit-gating
mechanism that improves unsupported-claim retractability across DAPO and
VPPO while preserving or improving visual sensitivity, benchmark accuracy,
and HallusionBench performance.
\end{itemize}

%% file: section/2_related_work.tex
\section{Related Work}
\label{sec:related_work}

\paragraph{Multimodal Reinforcement Learning with Verifiable Rewards.}
RLVR improves language-model reasoning by optimizing policies against
verifiable outcome rewards, with GRPO and its large-scale variants estimating
advantages from group-relative rewards
\citep{shao2024deepseekmath,guo2025deepseek,yu2026dapo}.
Recent work extends RLVR to vision--language models through task-specific
rewards and multimodal cold starts \citep{liu2025visual,huang2025visionr1},
multi-stage perception--reasoning training
\citep{meng2025mm,peng2025lmmr1,feng2026rewardmap},
and iterative SFT--RL optimization \citep{zhang2025r1,deng2025openvlthinker}.
Most of these methods still broadcast a single answer-level advantage across
all response tokens.
We show that this lets unsupported visual claims inside correct trajectories
inherit positive credit, and that the resulting persistence shift depends on
the optimizer: DAPO and VPPO worsen it, whereas GRPO does not have this effect.

\paragraph{Perception-Aware and Fine-Grained Policy Optimization.}
To make credit more visually grounded, prior work introduces explicit
perception rewards \citep{xiao2026perception}, implicit perception objectives
from visually suppressed inputs \citep{wang2025papo}, token-level reweighting
or selection of visually dependent tokens
\citep{huang2026spotlight,jiao2026credit}, and counterfactual or role-aware
objectives against language priors \citep{jiang2026srpo,yu2026cfpo}.
Multimodal process reward models further supervise individual steps
\citep{wang2026visualprm400k,zhang2026gm}, and grounding-oriented methods
tie reasoning to localized evidence regions
\citep{cao2025ground,huang2026evidence}.
These methods control where and how strongly the policy relies on the image.
PACG addresses a complementary question, namely whether a specific visual
claim loses support when its evidence is disrupted, and it can be composed
directly with perception-aware methods such as VPPO.
\paragraph{Multimodal Hallucination and Counterfactual Faithfulness.}
Visual hallucination is typically evaluated with object, attribute, and
relation benchmarks \citep{li2023pope,wang2023amber,guan2024hallusionbench}
and mitigated through preference optimization or fine-grained feedback
\citep{yu2024rlhfv,zhou2024povid,sarkar2025halva,yu2025rlaif}, or at inference
time through contrastive decoding \citep{leng2024vcd,huang2024opera}.
Counterfactual interventions have also been used to analyze modality-specific
causal effects, hallucination, and explanation faithfulness
\citep{li2025treble,ding2025edct,zafar2026medicalgrounding}.
These approaches either evaluate a fixed model or modify decoding and
preference data. In contrast, we use counterfactual image intervention inside RL training, turning
the retractability of each reasoning claim into a token-level credit signal, as this training-only mechanism adds no inference-time computational or latency cost.

%% file: figure/pacg_overview.tex
\begin{figure}[!t]
    \centering
    \includegraphics[width=\linewidth]{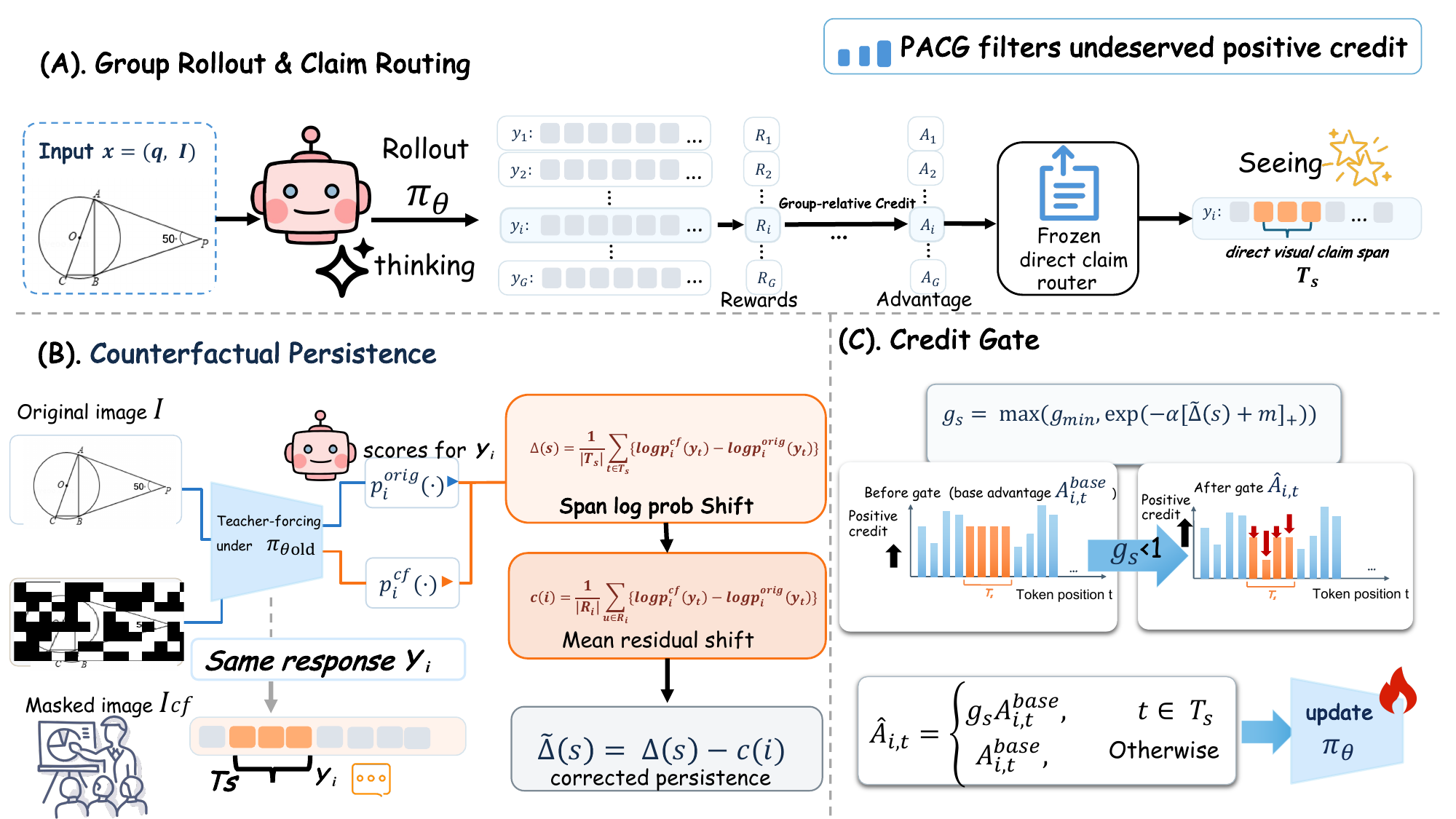}
    \caption{\textbf{Overview of PACG.}
    A frozen router identifies direct visual claims in responses scored
    under the original and degraded images. The residual-corrected signed
    shift $\widetilde{\Delta}(s)$ measures claim persistence and determines
    a soft gate on positive token credit. Negative advantages and
    non-routed tokens are unchanged, so the method only modulates positive credit for routed visual claims.}
    \label{fig:pacg_pipeline}
\end{figure}

%% file: section/3_preliminaries.tex
\section{Preliminaries}
\label{sec:preliminaries}

\paragraph{Group-relative RLVR.}
Given a multimodal prompt $x=(q,I)$, the rollout policy
$\pi_{\theta_{\mathrm{old}}}$ samples a group of responses
$\{y_i\}_{i=1}^{G}$, and each response receives a verifiable reward
$r_i\in\{0,1\}$ for final-answer correctness. GRPO and DAPO turn these
rewards into group-relative advantages
\citep{shao2024deepseekmath,yu2026dapo} and optimize a clipped objective,
written here in DAPO's token-mean form:
\begin{equation}
\mathcal{J}(\theta)=\mathbb{E}\!\left[
\frac{1}{\sum_i |y_i|}\sum_{i=1}^{G}\sum_{t=1}^{|y_i|}
\min\!\left(\rho_{i,t}A_{i,t},\,
\operatorname{clip}(\rho_{i,t},1-\epsilon_{\mathrm{low}},
1+\epsilon_{\mathrm{high}})A_{i,t}\right)\right],
\label{eq:clipped_objective}
\end{equation}
where $\rho_{i,t}=\pi_\theta(y_{i,t}\mid x,y_{i,<t})/
\pi_{\theta_{\mathrm{old}}}(y_{i,t}\mid x,y_{i,<t})$ denotes the
importance ratio and $A_{i,t}$ is the credit assigned to token $t$.
GRPO differs only in averaging the loss per sequence and in using
symmetric clipping, and we omit parent-specific regularizers.

\paragraph{Outcome-level credit ambiguity.}
What matters for this work is how the token credit $A_{i,t}$ is obtained.
Outcome-driven optimizers such as GRPO and DAPO simply broadcast the sequence
advantage to every token, assigning each token the same sequence-level advantage:
\begin{equation}
    A^{\mathrm{base}}_{i,t}=A_i,
    \label{eq:credit_broadcast}
\end{equation}
while perception-aware optimizers such as VPPO reshape this credit with
image-perturbation dependency scores \citep{huang2026spotlight}. In both
cases, the credit a visual claim receives depends on whether the final answer
is correct and on how strongly its tokens depend on the image, but not on
whether the image actually supports the claim. An unsupported claim inside a
correct trajectory therefore inherits positive credit, and the objective
provides no signal to distinguish it from a supported one. PACG is designed to
supply such a signal without requiring support labels.

%% file: section/4_method.tex
\section{Method}
\label{sec:method}

As illustrated in
Figure~\ref{fig:pacg_pipeline}, PACG proceeds in three steps. It first identifies the spans of a response
that make direct visual claims, then measures whether each claim retracts
when its visual evidence is disrupted, and finally scales down the positive
credit of claims that persist. Only the positive parent credit on routed claims
is modified during optimization, while all other credit remains unchanged.

\subsection{Direct Visual Claim Routing}
\label{sec:claim_routing}

We focus on direct visual claims, the minimal spans of a response that assert
information read directly from the image, such as an object attribute, a count,
an OCR value, a spatial relation, a chart reading, or a geometric property. As
shown in Figure~\ref{fig:pacg_pipeline}(A), this definition is functional
rather than factual. An incorrect assertion is still routed if it presents
itself as image evidence, whereas derived reasoning and paraphrases of the
question are excluded. A frozen \mbox{Qwen3.6-35B-A3B} router maps each claim
$s$ to its token set $\mathcal{T}_s$ in response $y_i$. After excluding routed
and otherwise visually dependent tokens, the remaining tokens form a residual
set $\mathcal{R}_i$ that serves as a within-response reference. Whether a claim
is supported by the image is annotated only for analysis and never used in
training.

\subsection{Fixed-Rollout Counterfactual Measurement}
\label{sec:cf_measurement}

To test whether a claim depends on its evidence, we degrade the image,
$I_i^{\mathrm{cf}}=\mathcal{C}(I_i)$, and teacher-force the same sampled
response under both images with a frozen policy $\pi_\phi$, as shown in
Figure~\ref{fig:pacg_pipeline}(B). Because every response prefix is fixed,
any change in the predictions of the model comes from the visual context
alone. We write $p^{\mathrm{orig}}_{i,t}(\cdot)=
\pi_\phi(\cdot\mid q_i,I_i,y_{i,<t})$ and
$p^{\mathrm{cf}}_{i,t}(\cdot)=
\pi_\phi(\cdot\mid q_i,I_i^{\mathrm{cf}},y_{i,<t})$, where
$\phi=\theta_{\mathrm{old}}$ during training and $\phi$ is the evaluated
checkpoint in offline analysis.

\paragraph{Sensitivity.}
A natural first measure is how much the predictive distribution at the claim
tokens changes. We define this change as Evidence-Function Sensitivity, or EFS for short.
\begin{equation}
    \operatorname{EFS}_\phi(s)
    =
    \frac{1}{|\mathcal{T}_s|}
    \sum_{t\in\mathcal{T}_s}
    D_{\mathrm{KL}}\!\left(
        p^{\mathrm{orig}}_{i,t}\,\|\,p^{\mathrm{cf}}_{i,t}
    \right).
    \label{eq:efs}
\end{equation}
EFS, however, is unsigned. A large divergence can arise when the claim
becomes less likely, but also when it becomes more likely, for instance when
the masked image leads the model to fall back on a language prior that favors
the same assertion. EFS therefore reveals how strongly a claim reacts to the
image but not whether the model withdraws it, and we use it only for
diagnosis.

\paragraph{Persistence.}
We track the claim's token log-likelihood to capture this directional change:
\begin{equation}
    \Delta_\phi(s)
    =
    \frac{1}{|\mathcal{T}_s|}
    \sum_{t\in\mathcal{T}_s}
    \left[
        \log p^{\mathrm{cf}}_{i,t}(y_{i,t})
        -\log p^{\mathrm{orig}}_{i,t}(y_{i,t})
    \right].
    \label{eq:raw_claim_shift}
\end{equation}
Image corruption also shifts the likelihood of the entire response, so a
negative $\Delta_\phi(s)$ may reflect trajectory-wide drift rather than a
claim-specific retraction. We therefore correct for this drift by subtracting the mean shift of the residual tokens in the same response:
\begin{equation}
    c_\phi(i)=
    \frac{1}{|\mathcal{R}_i|}
    \sum_{u\in\mathcal{R}_i}
    \left[
        \log p^{\mathrm{cf}}_{i,u}(y_{i,u})
        -\log p^{\mathrm{orig}}_{i,u}(y_{i,u})
    \right],
    \qquad
    \widetilde{\Delta}_\phi(s)=\Delta_\phi(s)-c_\phi(i).
    \label{eq:corrected}
\end{equation}
A negative $\widetilde{\Delta}_\phi(s)$ means that the claim loses more
support than the rest of the response, which we call retraction, whereas
values near or above zero indicate persistence. The residual set is a
practical reference, and the
intervention tests how a claim responds to its evidence. We refer to
the mismatch between EFS and corrected persistence as Sensitivity--Persistence Decoupling (SPD).

\subsection{Persistence-Aware Credit Gating}
\label{sec:pacg}

As shown in Figure~\ref{fig:pacg_pipeline}(C), the corrected persistence of 
each routed claim soft-gates:

\begin{equation}
    v_s=[\widetilde{\Delta}_{\theta_{\mathrm{old}}}(s)+m]_+,
    \qquad
    g_s=\max\!\left(g_{\min},e^{-\alpha v_s}\right),
    \label{eq:span_gate}
\end{equation}
where $m\geq0$ is the retraction margin, $\alpha\geq0$ controls the strength
of attenuation, and $g_{\min}\in(0,1]$ is the credit floor. Claims that
retract by at least $m$ keep their full positive credit, while more
persistent claims are attenuated smoothly toward $g_{\min}$, so even the most
persistent claim retains part of its credit. The gate is computed with the
frozen rollout policy and treated as a stop-gradient, which prevents the
policy from raising its credit by differentiating through the gate.

\subsection{PACG Policy Optimization}
\label{sec:pacg_optimization}

Finally, the gates are applied to the token credit of the parent. Let
$g^{\mathrm{span}}_{i,t}$ be the minimum gate among routed claims covering
token $t$, with value one for non-routed tokens. PACG replaces the original unscaled parent credit with a soft-gated credit computed as follows:
\begin{equation}
    \widetilde{A}_{i,t}
    =
    \begin{cases}
        g^{\mathrm{span}}_{i,t}A^{\mathrm{base}}_{i,t},
        & A^{\mathrm{base}}_{i,t}>0 \text{ and } t \text{ is routed},\\
        A^{\mathrm{base}}_{i,t},
        & \text{otherwise},
    \end{cases}
    \label{eq:pacg_advantage}
\end{equation}
and substitutes $\widetilde{A}_{i,t}$ for $A_{i,t}$ in
Equation~\ref{eq:clipped_objective}, leaving the reward, rollout generation,
normalization, and clipping unchanged. For GRPO and DAPO,
$A^{\mathrm{base}}_{i,t}$ is the broadcast credit of
Equation~\ref{eq:credit_broadcast}, and for VPPO it is the perception-shaped
token credit.

\paragraph{Why positive-only gating?}
Negative advantages already suppress tokens from failed trajectories,
including their persistent claims. Attenuating them would weaken this
corrective signal without establishing that any individual claim is wrong, so
PACG gates only positive credit. Because
$g^{\mathrm{span}}_{i,t}\in[g_{\min},1]$, PACG preserves the sign of every
credit, never increases its magnitude, and reduces to the parent optimizer
when no claim is routed or when all routed claims retract.

\paragraph{Can the policy evade the gate?}
A policy might try to avoid attenuation by producing fewer visual claims.
A free-generation audit shows only a modest drop in the rate of direct claims
together with longer responses, rather than a collapse of visual content.
Routing and fallback rules, this audit, the formal properties of the gate,
and the complete training procedure are provided in the appendix.

%% file: section/5_experiments.tex
\section{Experiments}
\label{sec:experiments}

\subsection{Experimental Setup}
\label{sec:experimental_setup}
\paragraph{Models and training.}
Following VPPO, we train Qwen2.5-VL-7B on ViRL39K~\citep{wang2025vl},
comparing Base, GRPO, DAPO, VPPO, and PACG compositions.
We also evaluate Qwen2.5-VL-32B and Qwen3-VL-2B-Thinking.
Public baselines include ThinkLite-VL, VL-Rethinker, NoisyRollout,
R1-ShareVL, and MM-Eureka
\citep{wang2025sota,wang2025vl,liu2025noisyrollout,yao2025r1,meng2025mm}.
The 7B and 32B runs use two epochs, learning rate $10^{-6}$, rollout batch
size 384, and response limits of 2,048 and 4,096 tokens, while the 2B runs
use 200 steps and 8,192 tokens. All runs use an entropy penalty of 0.1, and
PACG uses $m=0.05$, $\alpha=10$, $g_{\min}=0.30$, and random $14\times14$
patch blackening with probability 0.5. 
% Matched Qwen2.5-VL-7B runs are trained
% with three random seeds and we report their mean.
% TODO: 按实际情况写明 32B 和 2B 是否为单 seed，例如：
% The 32B and 2B runs use a single seed.
Further training details are provided in Appendix~\ref{app:training_details}.

\vspace{-8pt}
\paragraph{Evaluation.}
Task-level accuracy is evaluated on MMK12~\citep{meng2025mm},
MathVerse~\citep{zhang2024mathverse}, DynaMath~\citep{zou2024dynamath},
MathVision~\citep{wang2024measuring}, Geometry3K~\citep{lu2021inter},
We-Math~\citep{qiao2024we}, LogicVista~\citep{xiao2024logicvista},
MMMU-Pro~\citep{yue2024mmmu}, and Clever-Count~\citep{johnson2017clevr}.
We follow the evaluation protocol of VPPO~\citep{huang2026spotlight} and
report mean accuracy over eight responses at temperature 1.0, with Avg.\
denoting the unweighted mean over the nine benchmarks. PACG adds no inference
cost and increases the 7B training time per step by 57.3\%, as detailed in Appendix~\ref{app:coverage_efficiency}.
\vspace{-8pt}
\subsection{Main Results}
\label{sec:main_results}
\begin{table*}[t]
    \caption{\textbf{Main results across multimodal reasoning benchmarks
    and backbone families.} All entries are accuracy (\%). The first six
    datasets assess general mathematical and geometric reasoning, whereas
    the last three emphasize vision-dependent multimodal reasoning.
    Public 7B baselines are listed first, followed by our matched
    Qwen2.5-VL-7B runs. \textbf{Bold} and \underline{underlined} values
    denote the best and second-best completed results within each backbone
    block. Additional blocks show transfer to a larger model scale and
    a newer backbone family.}
    \label{tab:main_results}
    \vspace{-5pt}
    \centering
    \fontsize{7.5}{9.0}\selectfont
    \setlength{\tabcolsep}{1.9pt}
    \renewcommand{\arraystretch}{1.30}
    \begin{tabular}{@{}ll*{10}{c}@{}}
        \toprule
        \multirow{2}{*}{\textbf{Backbone}}
        & \multirow{2}{*}{\textbf{Method}}
        & \multicolumn{6}{c}{\textbf{General Mathematical \& Geometric Reasoning}}
        & \multicolumn{3}{c}{\textbf{Vision-Dependent Reasoning}}
        & \multirow{2}{*}{\textbf{Avg.}} \\
        \cmidrule(lr){3-8}\cmidrule(lr){9-11}
        & & \textbf{MMK12} & \textbf{\shortstack{Math\\Verse}}
        & \textbf{\shortstack{Dyna\\Math}} & \textbf{\shortstack{Math\\Vision}}
        & \textbf{\shortstack{Geometry\\3K}} & \textbf{\shortstack{We-\\Math}}
        & \textbf{\shortstack{Logic\\Vista}} & \textbf{\shortstack{MMMU-\\Pro}}
        & \textbf{\shortstack{Clever-\\Count}} & \\
        \midrule
        \multirow{11}{*}{\shortstack[l]{Qwen2.5-VL\\7B}}
        & ThinkLite-VL
        & 62.5 & 63.8 & 62.0 & \underline{31.5} & 35.8 & 66.4 & 40.0 & 27.6 & 77.8 & 51.9 \\
        & VL-Rethinker
        & 69.3 & 68.8 & 65.7 & 29.5 & 40.7 & 68.5 & 45.8 & 39.7 & 82.0 & 56.7 \\
        & NoisyRollout$^{*}$
        & 50.0 & 67.8 & 62.1 & 22.1 & \textcolor{gray}{46.9} & \underline{71.0} & 45.3 & 34.5 & 85.7 & 53.9 \\
        & R1-ShareVL
        & 70.9 & 68.2 & 63.9 & 27.5 & 41.2 & 69.9 & 45.4 & 35.1 & 81.2 & 55.9 \\
        & MM-Eureka
        & 67.5 & 65.4 & 64.8 & 27.1 & 40.8 & 65.5 & 45.8 & 35.3 & 77.8 & 54.4 \\
        \cmidrule(lr){2-12}
        & Base
        & 42.8 & 40.2 & 55.0 & 18.8 & 35.5 & 50.2 & 40.6 & 30.0 & 61.6 & 41.6 \\
        & GRPO
        & 72.2 & 67.8 & 65.1 & 30.1 & 42.0 & 67.8 & 47.1 & 38.3 & 81.0 & 56.8 \\
        & DAPO
        & 81.0 & 66.6 & 64.6 & 30.3 & 42.4 & 67.9 & 46.2 & 39.2 & 84.6 & 58.1 \\
        & VPPO
        & 80.4 & \underline{70.1} & \underline{66.5} & 31.0 & \underline{44.7}
        & 70.1 & \underline{47.9} & 40.0 & 87.4 & 59.8 \\
        \rowcolor{oursblue}
        & \textbf{DAPO+\pacg{} (Ours)}
        & \underline{81.2} & \textbf{70.6} & 65.9 & 30.6 & 44.3
        & 70.1 & 47.1 & \textbf{40.7} & \underline{88.7}
        & \underline{59.9} \\
        \rowcolor{oursbluestrong}
        & \textbf{VPPO+\pacg{} (Ours)}
        & \textbf{81.5} & \underline{70.1} & \textbf{67.0}
        & \textbf{32.1} & \textbf{45.7} & \textbf{71.5}
        & \textbf{50.4} & \underline{40.6} & \textbf{89.5}
        & \textbf{60.9} \\
        \midrule[1.0pt]
        \rowcolor{sectiongray}
        \multicolumn{12}{c}{\textit{Transfer to a Larger Model Scale}} \\
        \multirow{4}{*}{\shortstack[l]{Qwen2.5-VL\\32B}}
        & Base
        & 56.9 & 67.2 & 69.9 & 36.9 & 50.0 & 68.2 & 55.6 & 45.8 & 78.4 & 58.8 \\
        & GRPO
        & 79.5 & 74.1 & 65.0 & 41.0 & 53.4 & \textbf{73.3} & 58.6 & \textbf{48.6} & 83.5 & 64.1 \\
        & DAPO
        & \underline{86.5} & \underline{75.2} & \textbf{74.3}
        & \underline{41.8} & \underline{53.8} & 72.1 & \underline{59.1}
        & \underline{47.9} & \underline{86.2} & \underline{66.3} \\
        \rowcolor{oursblue}
        & \textbf{DAPO+\pacg{} (Ours)}
        & \textbf{87.1} & \textbf{75.9} & \underline{73.8}
        & \textbf{42.7} & \textbf{54.1} & \underline{72.8} & \textbf{61.1}
        & \underline{47.9} & \textbf{89.0} & \textbf{67.2} \\
        \midrule[1.0pt]
        \rowcolor{sectiongray}
        \multicolumn{12}{c}{\textit{Transfer to a Newer Backbone Family}} \\
        \multirow{6}{*}{\shortstack[l]{Qwen3-VL\\2B\\Thinking}}
        & Base
        & 31.1 & 51.5 & 53.1 & 19.0 & 36.4 & 57.2 & 32.0 & 20.4 & 74.1 & 41.6 \\
        & GRPO
        & 51.3 & 53.5 & 62.0 & 24.4 & 45.4 & 57.1 & 33.2 & 23.3 & 83.2 & 48.1 \\
        & DAPO
        & 50.8 & 60.1 & 61.5 & 25.1 & 44.2 & 65.6 & 43.7 & 24.9 & 86.2 & 51.3 \\
        & VPPO
        & \textbf{51.7} & \underline{68.9} & \textbf{62.3} & \textbf{26.7}
        & \underline{47.2} & 67.6 & 44.5 & 25.7 & 88.6 & \underline{53.7} \\
        \rowcolor{oursblue}
        & \textbf{DAPO+\pacg{} (Ours)}
        & \underline{51.6} & 67.5 & 61.8 & \underline{25.8} & 46.9
        & \underline{67.9} & \underline{44.6} & \underline{26.0}
        & \underline{89.0} & 53.5 \\
        \rowcolor{oursbluestrong}
        & \textbf{VPPO+\pacg{} (Ours)}
        & \textbf{51.7} & \textbf{69.4} & \underline{62.1} & \textbf{26.7}
        & \textbf{48.9} & \textbf{68.2} & \textbf{45.0} & \textbf{26.1}
        & \textbf{92.8} & \textbf{54.5} \\
        \bottomrule
    \end{tabular}
    \par\smallskip
    \begin{minipage}{\linewidth}
    \fontsize{8}{9.5}\selectfont
   
    $^{*}$NoisyRollout is trained using the training set of Geometry3K;
    its Geometry3K score is excluded from that column's ranking.
    \end{minipage}
\end{table*}

\paragraph{Q1: Does persistence-aware gating improve task accuracy?}
Yes, for both parent optimizers. As shown in Table~\ref{tab:main_results},
DAPO+\pacg{} improves all nine benchmarks over DAPO and raises the average
from 58.1 to 59.9, while VPPO+\pacg{} improves eight benchmarks and reaches
the best 7B average of 60.9. The largest gains appear on Clever-Count and
MathVerse, the benchmarks that depend most directly on reading the image.
 Because PACG leaves the reward untouched, these gains come from credit
allocation alone. Both average gains are statistically significant across three training seeds
under a two-sided Welch's $t$-test ($p<0.05$), and the per-seed results are
reported in Appendix~\ref{app:seed_variance}.\looseness=-1\par

\paragraph{Q2: Is the effect specific to one model?}
No. As shown in Table~\ref{tab:main_results}, DAPO+\pacg{} raises the average from 66.3 to 67.2 on Qwen2.5-VL-32B and
from 51.3 to 53.5 on Qwen3-VL-2B-Thinking, where VPPO+\pacg{} again gives the
best aggregate (54.5).

\begin{wraptable}[10]{r}{0.40\linewidth}
\vspace{-18pt}
\centering
\captionsetup{font=footnotesize,skip=3pt}
\caption{\textbf{HallusionBench} (7B; \%; single checkpoint per model).
}
\label{tab:hallusion}
\fontsize{8.5}{10}\selectfont
\setlength{\tabcolsep}{3.5pt}
\begin{tabular}{@{}lccc@{}}
\toprule
Model & aAcc$\uparrow$ & fAcc$\uparrow$ & qAcc$\uparrow$\\
\midrule
Base & 65.19 & 38.15 & 33.41\\
GRPO & 66.90 & 44.35 & 35.10\\
DAPO & 67.17 & 43.77 & 35.32\\
VPPO & 67.70 & 45.06 & 35.18\\
\midrule
\rowcolor{oursblue}
DAPO+\pacg{} & 67.64 & 44.64 & 35.54\\
\rowcolor{oursbluestrong}
VPPO+\pacg{} & \textbf{68.11} & \textbf{45.51} & \textbf{36.04}\\
\bottomrule
\end{tabular}
\vspace{-10pt}
\end{wraptable}
\paragraph{Q3: Does the gain carry over to a hallucination benchmark?}
Yes, consistently though modestly. Table~\ref{tab:hallusion} shows that
\pacg{} improves all three HallusionBench~\citep{guan2024hallusionbench}
metrics for both parents, and VPPO+\pacg{} is best on every column. The
largest gain for VPPO is 0.86 points on qAcc, which requires the same
question to be answered correctly across its visual variants. fAcc, which
requires every question about an image to be correct, also rises for both
parents, so the gains are not confined to isolated questions.
\par
\WFclear
\begin{figure}[!htbp]
    \centering
    \includegraphics[width=0.62\linewidth]{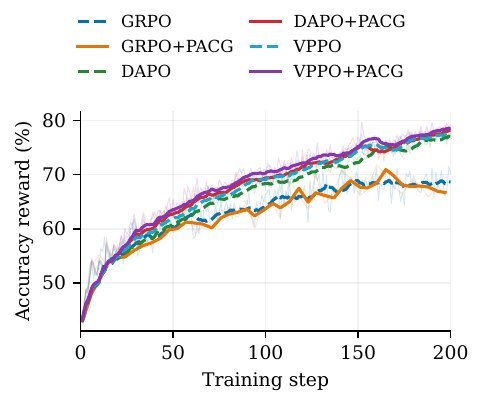}
    \captionsetup{skip=2pt}
    \caption{\textbf{Online accuracy reward.} Dashed: parents;
    solid: PACG compositions; curves are trailing 10-step means for one
    representative seed per method.}
    \label{fig:train_reward}
\end{figure}

\paragraph{Q4: Does gating positive credit slow down outcome learning?}
No. As shown in Figure~\ref{fig:train_reward}, PACG closely tracks the reward
trajectory of GRPO and raises the late-stage reward of DAPO and VPPO. Online
entropy, length, and gate statistics are reported in
Appendix~\ref{app:online_training_diagnostics}.

\subsection{Ablation Study}
\label{sec:gate_ablation}
\paragraph{Q5: Which components of the gate actually matter?}
All three. Table~\ref{tab:gate_ablation} shows that without direct-claim
routing, gating every positive token destroys the training signal and
accuracy collapses to 38.0. Removing residual calibration or the
positive-only restriction remains competitive on the mathematical benchmarks
but loses 7.4 and 6.1 points on Clever-Count and weakens unsupported-claim
retractability, confirming that both components are necessary.
\WFclear

\begin{table}[!htb]
    \centering
    \caption{\textbf{Component ablations with DAPO as parent.}
    \emph{w/o Residual Corr.}: omit same-rollout residual calibration;
    \emph{w/o Direct Routing}: gate all positive tokens instead of direct claims;
    \emph{w/o Positive-Only}: also attenuate negative advantages.
    Accuracy and its nine-benchmark mean (Avg.) are percentages.
    Unsup. Pers. is the unsupported corrected persistence defined in Section~\ref{sec:mechanism_results}.
    Entries are point estimates; bold marks the best column value.}
    \label{tab:gate_ablation}
    \label{tab:gate_ablation_acc}
    \vspace{-3pt}
    {\fontsize{8.0}{9.6}\selectfont
    \setlength{\tabcolsep}{3.5pt}
    \renewcommand{\arraystretch}{1.22}
    \resizebox{\linewidth}{!}{%
    \begin{tabular}{@{}l*{11}{c}@{}}
        \toprule
        \textbf{Variant}
        & \textbf{MMK12} & \shortstack{\textbf{Math}\\\textbf{Verse}}
        & \shortstack{\textbf{Dyna}\\\textbf{Math}}
        & \shortstack{\textbf{Math}\\\textbf{Vision}}
        & \shortstack{\textbf{Geometry}\\\textbf{3K}}
        & \shortstack{\textbf{We-}\\\textbf{Math}}
        & \shortstack{\textbf{Logic}\\\textbf{Vista}}
        & \shortstack{\textbf{MMMU-}\\\textbf{Pro}}
        & \shortstack{\textbf{Clever-}\\\textbf{Count}}
        & \textbf{Avg.} & \shortstack{\textbf{Unsup.}\\\textbf{Pers.} $\downarrow$} \\
        \midrule
        DAPO
        & 81.0 & 66.6 & 64.6 & 30.3 & 42.4 & 67.9
        & 46.2 & 39.2 & 84.6 & 58.1 & $-0.1241$ \\
        w/o Residual Corr.
        & 80.9 & 69.0 & 65.6 & 30.9 & \textbf{45.4} & 69.1
        & \textbf{48.9} & 40.1 & 81.3 & 59.0 & $-0.1278$ \\
        w/o Direct Routing
        & 46.3 & 40.8 & 36.7 & 16.2 & 10.4 & 44.0
        & 31.7 & 30.3 & 85.2 & 38.0 & $-0.0374$ \\
        w/o Positive-Only
        & 80.8 & 70.3 & \textbf{66.1} & \textbf{31.3}
        & 44.5 & \textbf{70.6} & 47.7 & 39.6 & 82.6
        & 59.3 & $-0.1301$ \\
        \rowcolor{oursblue}
        \textbf{PACG (full)}
        & \textbf{81.2} & \textbf{70.6} & 65.9 & 30.6 & 44.3
        & 70.1 & 47.1 & \textbf{40.7} & \textbf{88.7}
        & \textbf{59.9} & $\mathbf{-0.1392}$ \\
        \bottomrule
    \end{tabular}}}
\end{table}

\Needspace{11\baselineskip}
\begin{wraptable}[9]{r}{0.46\linewidth}
\vspace{-12pt}
\centering
\captionsetup{font=footnotesize,skip=4pt}
\caption{\textbf{Margin sweep} ($\alpha{=}10$, $g_{\min}{=}0.30$).
Acc.: MMK12; Ret.: positive-credit retention (both \%).}
\label{tab:margin}
\small
\setlength{\tabcolsep}{3.5pt}
\begin{tabular}{@{}ccccc@{}}
\toprule
$m$ & Acc. & EFS $\uparrow$ & \shortstack{Unsup.\\Pers. $\downarrow$} & Ret.\\
\midrule
0.00 & 81.0 & \textbf{0.34} & $-0.08$ & 95\\
\rowcolor{oursblue}
0.05 & \textbf{81.2} & 0.31 & $-0.14$ & 81\\
0.10 & 81.0 & 0.28 & $\mathbf{-0.20}$ & 65\\
\bottomrule
\end{tabular}
\vspace{-10pt}
\end{wraptable}
\paragraph{Margin selection.}
Table~\ref{tab:margin} compares three margins. The zero margin retains 95\%
of positive credit and matches the MMK12 accuracy of ungated DAPO. At
$m=0.10$, retention drops to 65\% and EFS to 0.28, and accuracy falls
slightly below the result at $m=0.05$ despite stronger retraction. We
therefore use $m=0.05$ as the operating point.
\WFclear

\subsection{Sensitivity--Persistence Decoupling and Its Mitigation}
\label{sec:mechanism_results}
\paragraph{Protocol.}
CrossBench4-2000 contains 2,000 questions from Geometry3K, MMK12, LogicVista,
and MathVerse, each with four fixed Base responses, and the questions are
selected without regard to correctness or span labels. Because teacher-forced
scoring of every checkpoint is costly, mechanism results use the model from
the first of the three training seeds. Each model scores the same responses
and spans under the original image and a hash-fixed masked image. We report
direct EFS from Equation~\ref{eq:efs}, where higher values mean greater visual
sensitivity, and corrected persistence from Equation~\ref{eq:corrected}, where
lower values mean stronger retraction. More details are in Appendix~\ref{app:crossbench_reference}.

\begin{figure}[!htbp]
\centering
\includegraphics[width=\linewidth]{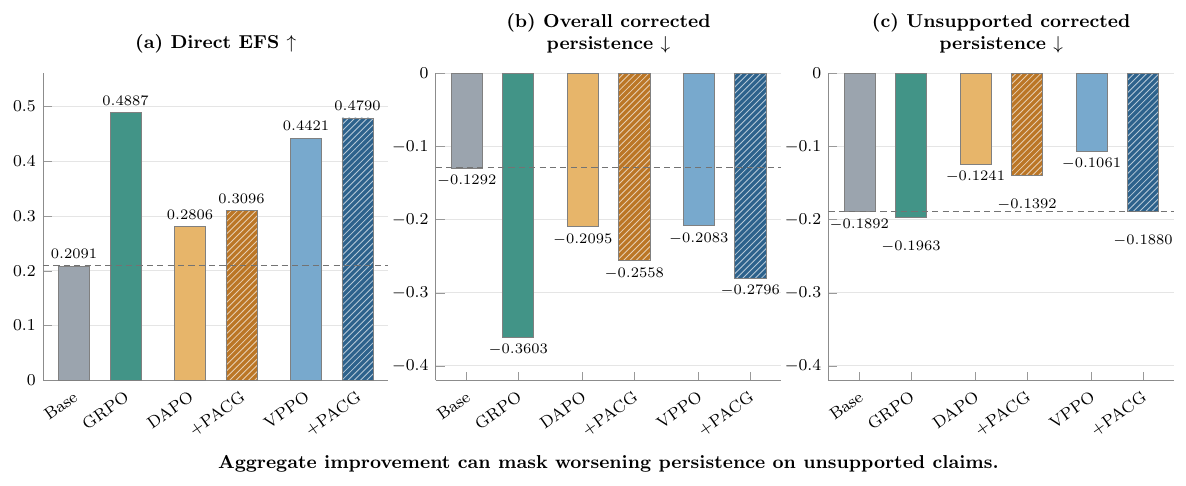}
\caption{\textbf{Aggregate improvement masks unsupported-claim deterioration.}
(a) Direct EFS ($\uparrow$); (b) overall and (c) unsupported corrected
persistence ($\downarrow$). DAPO and VPPO improve aggregate retractability
while worsening unsupported persistence; PACG mitigates this deterioration.
 }
\label{fig:aggregate_unsupported_persistence}
\end{figure}

\paragraph{Q6: Does higher visual sensitivity come with better retractability?}
No, and this is the central finding. As shown in
Figure~\ref{fig:aggregate_unsupported_persistence}, DAPO and VPPO both
increase EFS and reduce overall persistence, yet their unsupported
persistence rises by 0.0651 and 0.0831. GRPO increases EFS the most without
a resolved deterioration, so the two quantities are decoupled rather than
tied. DAPO and VPPO also apply roughly three times the effective positive
update mass of GRPO to routed claims, which may contribute to this observed difference and suggests that the effect is optimizer-dependent.

\Needspace{19\baselineskip}
\begin{wraptable}{r}{0.52\linewidth}
\vspace{-10pt}
\captionsetup{font=footnotesize,skip=4pt}
\centering
\caption{\textbf{Claim-level mechanism results.}
Accuracy (\%) is evaluated on the CrossBench4-2000 questions using seed-0 checkpoints.
DID is the Base-relative image-versus-text corrected-persistence contrast.
$^\dagger$ marks an unsupported-persistence change whose paired 95\%
question-clustered bootstrap CI excludes zero: versus Base for parent
optimizers, and versus the corresponding parent for PACG.
Full intervals are in Appendix~\ref{app:supported_persistence}.}
\label{tab:mechanism_results}
\small
\setlength{\tabcolsep}{3pt}
\renewcommand{\arraystretch}{1.15}
\resizebox{\linewidth}{!}{%
\begin{tabular}{@{}lccc@{}}
\toprule
Model & \shortstack{Unsup.\\Pers. $\downarrow$} & \shortstack{Img-vs-Text\\DID $\downarrow$} & \shortstack{Acc. (\%)\\$\uparrow$}\\
\midrule
Base & $-0.1892$ & 0.000 & 38.9\\
GRPO & $\mathbf{-0.1963}$ & $\mathbf{-0.167}$ & 59.0\\
DAPO & $-0.1241^\dagger$ & $-0.022$ & 62.2\\
VPPO & $-0.1061^\dagger$ & $-0.042$ & 63.4\\
\midrule
\rowcolor{oursblue}
DAPO+\pacg{} & $-0.1392^\dagger$ & $-0.016$ & 63.0\\
\rowcolor{oursbluestrong}
VPPO+\pacg{} & $-0.1880^\dagger$ & $-0.066$ & \textbf{64.3}\\
\bottomrule
\end{tabular}}
\vspace{-8pt}
\end{wraptable}

\paragraph{Q7: Does \pacg{} fix the decoupling it diagnoses?}
Partly for DAPO and substantially for VPPO. Table~\ref{tab:mechanism_results}
shows that PACG lowers unsupported persistence by 0.0151 and 0.0819 while
raising EFS and accuracy. VPPO+\pacg{} returns to the Base level, whereas
DAPO+\pacg{} recovers part of the gap. The phase trajectory in
Figure~\ref{fig:phase_trajectory} shows the same divergence over training.
DAPO drifts toward greater sensitivity but weaker retractability, while PACG
gains sensitivity without this drift. Image-versus-text DID worsens slightly
for DAPO+\pacg{} and improves for VPPO+\pacg{}, so this complementary
diagnostic depends on the optimizer. Further analyses are provided in
Appendix~\ref{app:supp_mechanism}.
\WFclear

\begin{figure}[!htbp]
\vspace{\dimexpr\baselineskip+14pt\relax}
\centering
\includegraphics[width=0.68\linewidth]{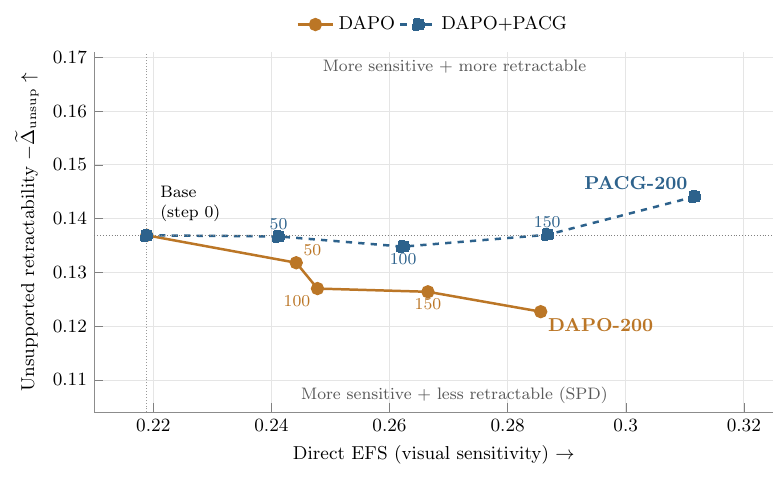}
\caption{\textbf{Sensitivity--retractability phase trajectory.}
Direct EFS versus negative unsupported persistence on the fixed
500-question subset. Right: more sensitive, up: more retractable.
Lines connect steps $0,50,100,150,200$; dotted lines mark Base.}
\label{fig:phase_trajectory}
\vspace{\dimexpr-\baselineskip+6pt\relax}
\end{figure}

%% file: section/8_conclusion.tex
\section{Conclusion}
\label{sec:conclusion}
In this paper, we identify Sensitivity--Persistence Decoupling as an
overlooked failure of credit assignment in multimodal RL. Using a
fixed-rollout counterfactual diagnostic, we show that DAPO and VPPO make
visual claims more sensitive to the image while simultaneously making
unsupported claims more persistent. To address this, we propose
Persistence-Aware Credit Gating (PACG), which tests whether each direct
visual claim retracts when its evidence is disrupted and selectively
attenuates the positive credit of claims that do not. Without
supported/unsupported labels or inference-time cost, PACG improves
unsupported-claim retractability, raises accuracy across nine multimodal
reasoning benchmarks for both outcome-driven and perception-aware optimizers,
and transfers to a larger scale, a newer backbone, and
HallusionBench. Our work shows that whether a policy looks at the image and
whether it would take back what the image does not support are fundamentally
distinct questions. We believe that holding visual claims accountable to
visual evidence, rather than only making them dependent on it, offers a
promising direction for multimodal credit assignment.

%% file: section/7_appendix.tex
% Appendix-only navigation follows the PAPO reference layout.
\renewcommand{\contentsname}{Appendix Contents}
\begingroup
\fontsize{11}{13}\selectfont
\color{pacgpromptblue}
\hypersetup{colorlinks=true,linkcolor=pacgpromptblue}
\setlength{\parskip}{0pt}
\makeatletter
\renewcommand{\l@section}[2]{\addpenalty{\@secpenalty}\vskip 4pt
  \@dottedtocline{1}{0em}{2.2em}{\bfseries #1}{\textcolor{pacgpromptblue}{#2}}}
\renewcommand{\l@subsection}[2]{\@dottedtocline{2}{2.2em}{3.0em}{#1}{\textcolor{pacgpromptblue}{#2}}}
\makeatother
\tableofcontents
\endgroup
\clearpage

% Keep manuscript typography; apply compact float spacing only in the appendix.
\setlength{\textfloatsep}{12pt plus 2pt minus 2pt}
\setlength{\floatsep}{10pt plus 2pt minus 2pt}
\setlength{\intextsep}{10pt plus 2pt minus 2pt}
\captionsetup{font=small,skip=5pt}
\renewcommand{\topfraction}{0.9}
\renewcommand{\bottomfraction}{0.8}
\renewcommand{\textfraction}{0.08}
\renewcommand{\floatpagefraction}{0.7}

\input{appendices/training_details}
\FloatBarrier
\input{appendices/pacg_procedure}
\FloatBarrier
\input{appendices/reference_protocol}

\FloatBarrier
\input{appendices/routing_audits}

\FloatBarrier
\input{appendices/mechanism_analysis}
\FloatBarrier
\input{appendices/task_results}

\FloatBarrier
\input{appendices/online_diagnostics}
\FloatBarrier
\input{appendices/efficiency}
\FloatBarrier
\input{appendices/prompts}
\input{appendices/limitations}
\clearpage
\input{appendices/qualitative_comparisons}

%% file: appendices/training_details.tex
\section{Training and Implementation Details}
\label{app:training_details}

\subsection{Shared Training Recipe}
We follow the VPPO training setup for Qwen2.5-VL, using ViRL39K,
two training epochs, and eight sampled responses per prompt. Qwen2.5-VL-32B
uses the same settings as the 7B model, including the PACG parameters,
except that its maximum response length is 4,096 instead of 2,048. The
Qwen3-VL-2B-Thinking runs use a 200-step budget and a maximum response
length of 8,192, with the same entropy penalty of 0.1. The vision
tower remains trainable. Rewards are binary answer correctness, without
an additional format reward. Unlike the VPPO reference setting's entropy
coefficient of 0.06, our reported recipe uses an entropy penalty of 0.1:
the minimized loss includes $+0.1\,\mathcal{H}$, rather than an entropy
bonus. Table~\ref{tab:pacg_training_config} summarizes the settings for all three backbones.
Qwen3-VL-2B-Thinking is the additional backbone evaluated in the main
results; its results are reported separately from Qwen2.5-VL-7B and Qwen2.5-VL-32B.

\begin{table}[!htbp]
\centering
\caption{\textbf{Training configurations.} All three backbones use an entropy penalty
of 0.1; training budgets and response-length limits are backbone-specific.}
\label{tab:pacg_training_config}
\footnotesize
\setlength{\tabcolsep}{4pt}
\renewcommand{\arraystretch}{1.15}
\begin{tabular}{@{}lccc@{}}
\toprule
\textbf{Setting} & \textbf{Qwen2.5-VL-7B} & \textbf{Qwen2.5-VL-32B} & \shortstack{\textbf{Qwen3-VL-2B}\\\textbf{Thinking}} \\
\midrule
Training corpus & ViRL39K & ViRL39K & ViRL39K \\
Training budget & 2 epochs & 2 epochs & 200 steps \\
Optimizer & AdamW & AdamW & AdamW \\
Learning rate & $10^{-6}$ & $10^{-6}$ & $10^{-6}$ \\
Adam betas & $(0.9,0.999)$ & $(0.9,0.999)$ & $(0.9,0.999)$ \\
Weight decay & 0.01 & 0.01 & 0.01 \\
Learning-rate schedule & \shortstack{Constant,\\no warmup} & \shortstack{Constant,\\no warmup} & \shortstack{Constant,\\no warmup} \\
Freeze vision tower & No & No & No \\
Rollout batch size & 384 & 384 & 384 \\
Global update batch size & 128 & 128 & 128 \\
Responses per prompt & 8 & 8 & 8 \\
Rollout temperature / top-$p$ & 1.0 / 0.99 & 1.0 / 0.99 & 1.0 / 0.99 \\
Maximum prompt length & 4,096 & 4,096 & 4,096 \\
Maximum response length & 2,048 & 4,096 & 8,192 \\
Entropy penalty coefficient & 0.1 & 0.1 & 0.1 \\
Training reward & Binary accuracy & Binary accuracy & Binary accuracy \\
Parameter precision & BF16 & BF16 & BF16 \\
Gradient checkpointing & Enabled & Enabled & Enabled \\
\bottomrule
\end{tabular}
\end{table}

\subsection{Parent Optimizers and Regularization}
\label{app:baselines}
DAPO and its PACG composition use dynamic sampling, asymmetric clipping
with lower and upper thresholds 0.20 and 0.28, and token-mean policy loss.
Plain GRPO does not inherit DAPO's dynamic filtering. PACG preserves the
chosen parent's reward and loss configuration and changes only the routed
positive token credits. Frozen-reference KL regularization is separate
from the entropy penalty and from VPPO's visual-dependency score. All run
configurations use frozen-reference KL regularization with
\texttt{low\_var\_kl} and coefficient 0.01.

\subsection{PACG Configuration}
We use retraction margin $m=0.05$, attenuation strength $\alpha=10$,
and minimum gate $g_{\min}=0.30$, with
$g_s=\max\{g_{\min},\exp[-\alpha\max(0,\widetilde{\Delta}(s)+m)]\}$.
Gating is restricted to routed direct visual claims and positive advantages;
negative advantages and non-routed tokens are unchanged. Overlapping spans
use the minimum gate. Residual calibration requires at least eight residual
response tokens; invalid routing, alignment, intervention, or residual
control falls back to unit gates. The training counterfactual is random
patch blackening with $14\times14$-pixel patches and blackening probability
0.5. Supported/unsupported labels are not used to assign gates.

\subsection{Composition with VPPO}
VPPO retains its visual-token selection fraction of 0.4 and minimum
trajectory-level advantage scaling of 0.9, with the upper scaling factor
determined dynamically by batch normalization. The visual-dependency score
uses the low-variance KL estimator under image perturbation. This score is
not a KL regularizer toward a frozen reference model. PACG is applied to
the resulting base credits with the same positive-only rule as for DAPO.

\subsection{Systems and Evaluation}
Training uses the repository's VERL/EasyR1-derived pipeline with FSDP,
BF16 parameters, gradient checkpointing, and vLLM rollout generation.
Micro-batching and rollout tensor parallelism are deployment settings,
not gate hyperparameters; the shared 7B launcher defaults to experience
micro-batch size 8 and update micro-batch size 4 per device. The measured 7B timing
configuration uses 16 PPU-810E accelerators
(Appendix~\ref{app:coverage_efficiency}). Evaluation uses
eight stochastic responses at temperature 1.0 and the same answer
extraction and verification pipeline across matched models. Evaluation
generation limits are distinct from the training response limits above.

%% file: appendices/pacg_procedure.tex
\section{Full PACG Training Procedure}
\label{app:pacg_algorithm}

Algorithm~\ref{alg:pacg} summarizes the training procedure.
PACG preserves the parent's rollout generation, reward computation, and
policy objective, modifying only eligible positive token credits.

\begin{algorithm}[!htbp]
\caption{Persistence-Aware Credit Gating for Group-Relative RLVR}
\label{alg:pacg}
\small
\begin{algorithmic}[1]
\Require Prompt $x=(q,I)$; frozen rollout policy
$\pi_{\theta_{\mathrm{old}}}$; base optimizer; frozen claim router;
degradation operator $\mathcal{C}$
\State Sample rollout group $\{y_i\}_{i=1}^{G}$ from
$\pi_{\theta_{\mathrm{old}}}(\cdot\mid x)$
\State Compute rewards and base token credits
$\{A^{\mathrm{base}}_{i,t}\}$ with the chosen optimizer
\State Initialize $\widetilde{A}_{i,t}\gets A^{\mathrm{base}}_{i,t}$
and $g_{i,t}^{\mathrm{span}}\gets 1$ for all response tokens
\For{$i=1,\ldots,G$}
    \State Route and token-align direct visual claims in $y_i$
    \State Construct residual response-token set $\mathcal{R}_i$
    \State Construct and validate $I_i^{\mathrm{cf}}=\mathcal{C}(I)$
    \If{routing, alignment, intervention, or residual control is invalid}
        \State Set all span gates for $y_i$ to one
        \State \textbf{continue}
    \EndIf
    \State Teacher-force fixed $y_i$ under $I$ and $I_i^{\mathrm{cf}}$
    using $\pi_{\theta_{\mathrm{old}}}$
    \State Compute $\widetilde{\Delta}(s)$ and $g_s$ for each routed claim
    \State Resolve overlapping claims using
    $g_{i,t}^{\mathrm{span}}=\min_{s\ni t}g_s$;
    retain one for tokens outside routed claims
    \State Form $\widetilde{A}_{i,t}$ with the positive-only rule in
    Equation~\ref{eq:pacg_advantage}
\EndFor
\State Update $\pi_\theta$ with the base clipped objective using
$\widetilde{A}_{i,t}$
\end{algorithmic}
\end{algorithm}

\subsection{Routing and Residual Calibration}
\label{app:routing_details}

The frozen router identifies contiguous spans asserting information read
directly from the image, including counts, OCR readings, geometric
properties, and spatial relations. Directness identifies a span's
evidential role; supported/unsupported labels are reserved for analysis.
For rollout $i$, let $\mathcal V_i$ contain tokens annotated as direct or
derived visual evidence, or as directly or indirectly visually dependent.
The residual control set is
\begin{equation}
\mathcal R_i=\{t:\ t\text{ is a valid response-text token and }
t\notin\mathcal V_i\}.
\end{equation}
PACG requires at least eight residual tokens. Invalid routing, alignment,
intervention, or residual control yields unit gates. The annotation
instructions and output schema are reproduced in
Appendix~\ref{app:router_annotation_prompt}.

For rollout $y_i$, a frozen router returns a character interval
$s=(a_s,b_s)$, which is aligned to the response-token set
\begin{equation}
    \mathcal{T}_s
    =
    \left\{
    t:
    [c_t^{\mathrm{start}},c_t^{\mathrm{end}})
    \cap
    [a_s,b_s)
    \neq\varnothing
    \right\},
    \label{eq:span_alignment}
\end{equation}
where
$[c_t^{\mathrm{start}},c_t^{\mathrm{end}})$
is token $t$'s decoded character interval.

The residual set is a practical within-trajectory reference, not a set
assumed to be strictly vision-independent. Its mean shift calibrates
trajectory-wide drift rather than estimating a pure non-visual baseline.
Routing is an eligibility decision, not an error detector: no
supported/unsupported label conditions the gate. Image-level degradation
likewise need not remove exactly the evidence relevant to a single claim.

\subsection{Formal Analysis}
\label{app:formal}
\label{sec:pacg_properties}
\label{app:pacg_properties}

We give two propositions and a remark about PACG's credit transformation.
They are conditional, fixed-update properties, not convergence guarantees.
Changes in future claim persistence are evaluated empirically in
Section~\ref{sec:mechanism_results}.

Write $g(\delta)=\max\{g_{\min},\exp(-\alpha[\delta+m]_+)\}$.
For $\alpha>0$ and $0<g_{\min}<1$, this function is non-increasing,
equals one for $\delta\leq-m$, and is strictly decreasing on
$(-m,\delta_{\mathrm f})$, where
$\delta_{\mathrm f}=\log(1/g_{\min})/\alpha-m$. On this interval,
\begin{equation}
g'(\delta)=-\alpha g(\delta)<0.
\label{eq:gate_monotonicity}
\end{equation}
Outside it the derivative is zero, apart from the boundary points.
If $\alpha=0$ or $g_{\min}=1$, the gate is identically one.

\paragraph{Conditional label-free credit redistribution.}
Fix a rollout policy and a sampling rule over valid, positively credited
routed claim tokens with a single covering claim. Let $A>0$ be the parent
token credit, $X$ the claim's corrected persistence, and
$Z\in\{\mathrm{sup},\mathrm{unsup}\}$ its evidential status.
All expectations and distributions below use this same positive-credit
population, not an unconditional collection of claims.
Set $C_z^{\mathrm{par}}=\mathbb E[A\mid Z=z]$ and
$C_z=\mathbb E[g(X)A\mid Z=z]$.

\begin{proposition}[Relative positive-credit retention]
\label{prop:redistribution}
Suppose both status groups have positive probability,
$0<C_z^{\mathrm{par}}<\infty$, and $A$ and $X$ are independent conditional
on $Z$. Suppose also that
$F_{\mathrm{unsup}}(x)\leq F_{\mathrm{sup}}(x)$ for all $x$,
where $F_z$ is the conditional CDF of $X$ given $Z=z$.
Then
\begin{equation}
\frac{C_{\mathrm{unsup}}}{C_{\mathrm{sup}}}
\leq
\frac{C_{\mathrm{unsup}}^{\mathrm{par}}}
     {C_{\mathrm{sup}}^{\mathrm{par}}}.
\end{equation}
For $\alpha>0$ and $g_{\min}<1$, the inequality is strict if
$F_{\mathrm{sup}}-F_{\mathrm{unsup}}>0$ on a set of positive
Lebesgue measure inside $(-m,\delta_{\mathrm f})$.
\end{proposition}
\begin{proof}
Conditional independence gives
$C_z=C_z^{\mathrm{par}}\mathbb E[g(X)\mid Z=z]$.
The gate is non-increasing, so the assumed stochastic dominance implies
$\mathbb E[g(X)\mid Z=\mathrm{unsup}]
\leq\mathbb E[g(X)\mid Z=\mathrm{sup}]$.
For $\alpha>0$ and $g_{\min}<1$, their difference in the opposite order is
\[
\int_{-m}^{\delta_{\mathrm f}}
\alpha e^{-\alpha(x+m)}
\bigl(F_{\mathrm{sup}}(x)-F_{\mathrm{unsup}}(x)\bigr)\,dx\geq0.
\]
The stated strictness condition makes this integral positive.
Since every gate is at least $g_{\min}>0$, division yields the result.
\end{proof}

No support label enters the gate. The conclusion concerns a relative
credit ratio, not an absolute ordering of credit between groups.
For overlapping claims the actual gate is the minimum of their gates;
the proposition applies if its assumptions are instead imposed on the
effective persistence $X=\max_{s\ni t}\widetilde\Delta(s)$ and an explicitly
defined token-status population. Dominance of individual claim
persistences alone does not establish dominance of this maximum.
Invalid-scoring fallbacks are excluded from Proposition~\ref{prop:redistribution}.

The assumptions concern the evolving rollout policy and are not
established by our fixed-Base-response evaluation. In that evaluation,
unsupported persistence is lower than supported persistence under Base
($-0.1892$ versus $-0.137$), whereas the mean ordering reverses under
DAPO ($-0.1241$ versus $-0.159$) and VPPO ($-0.1061$ versus $-0.184$;
Tables~\ref{tab:mechanism_results} and~\ref{tab:supported_persistence}).
The latter ordering is compatible with the proposed condition but proves
neither stochastic dominance nor conditional independence, and does not
verify either assumption on online positive-credit tokens.
The proposition therefore describes a possible operating regime rather
than an experimentally established explanation of PACG's gains.

\paragraph{Bounded deviation at a fixed actor update.}
\begin{proposition}[Credit bounds and unchanged surrogate branches]
\label{prop:bounded}
Hold the sampled tokens, loss masks, parent advantages, and gates fixed,
with no differentiation through advantages or gates. PACG satisfies
\begin{equation}
\operatorname{sign}(\widetilde A_{i,t})
=\operatorname{sign}(A^{\mathrm{base}}_{i,t}),\qquad
|\widetilde A_{i,t}-A^{\mathrm{base}}_{i,t}|
\leq(1-g_{\min})[A^{\mathrm{base}}_{i,t}]_+.
\label{eq:pacg_conservative}
\end{equation}
At the same policy parameters, PACG preserves the active branches of
the parent's clipped surrogate. Wherever its per-token gradient exists,
a gated token's surrogate gradient is multiplied by
$g^{\mathrm{span}}_{i,t}$. For the token-mean objective in
Equation~\ref{eq:clipped_objective}, with positive policy ratios and
$0\leq\epsilon_{\mathrm{low}}<1$, $\epsilon_{\mathrm{high}}\geq0$,
\begin{equation}
|\mathcal J_{\mathrm{PACG}}-\mathcal J_{\mathrm{par}}|
\leq(1-g_{\min})(1+\epsilon_{\mathrm{high}})
\mathbb E\!\left[
\frac{\sum_i\sum_{t\in\mathcal T_i^+} A^{\mathrm{base}}_{i,t}}
     {\sum_i |y_i|}
\right],
\label{eq:pacg_objective_bound}
\end{equation}
where $\mathcal T_i^+$ contains routed tokens with positive parent credit.
The expectations are assumed finite.
\end{proposition}
\begin{proof}
The credit bounds follow from $g^{\mathrm{span}}_{i,t}\in[g_{\min},1]$;
all other credits are unchanged. For $A>0$, write the surrogate as
$S(\rho,A)=A\min\{\rho,1+\epsilon_{\mathrm{high}}\}$.
Thus $S(\rho,gA)=gS(\rho,A)$ for a fixed positive gate: branch selection
is unchanged and differentiation scales the token gradient by $g$.
At branch ties, the corresponding one-sided derivatives scale likewise.
Furthermore,
\[
0\leq S(\rho,A)-S(\rho,gA)
=(1-g)A\min\{\rho,1+\epsilon_{\mathrm{high}}\}
\leq(1-g_{\min})A(1+\epsilon_{\mathrm{high}}).
\]
Summing with the objective's nonnegative normalization proves the bound.
\end{proof}

The same argument applies to sequence-mean GRPO with its corresponding
weights in the bound. If tokens are masked, the same fixed masks and
normalization must be used on both sides. Unchanged regularizers cancel
in an objective-value comparison, but their gradients are not multiplied
by the gate. Per-token scaling need not preserve the direction of the
summed gradient or the eventual optimizer step. These results do not
assert a hard trust region, unchanged future clipping rates, or training
stability. Unit gates recover the parent objective, including when no
claim is routed or validity checks fail.

\paragraph{Invariance of residual calibration.}
\begin{remark}[Common additive shifts cancel exactly]
\label{rem:residual}
For fixed claim and residual token sets with nonempty residual control,
suppose $d_{i,t}=\eta_{i,t}+b_i$, where $b_i$ is shared by all response
tokens in rollout $i$. Then Equation~\ref{eq:corrected} gives
\[
\widetilde\Delta(s)=
\frac{1}{|\mathcal T_s|}\sum_{t\in\mathcal T_s}\eta_{i,t}
-\frac{1}{|\mathcal R_i|}\sum_{u\in\mathcal R_i}\eta_{i,u},
\]
which is exactly independent of $b_i$, whereas the raw shift retains
$b_i$. This cancellation motivates residual calibration and the
w/o Residual Corr.\ ablation (Table~\ref{tab:gate_ablation}); it does not
remove heterogeneous token-specific corruption effects or certify
claim-localized evidence removal.
\end{remark}

%% file: appendices/reference_protocol.tex
\section{Fixed-Reference Evaluation Protocol}
\label{app:crossbench_reference}

\subsection{CrossBench4 Reference Construction}

CrossBench4 comprises Geometry3K, MMK12, LogicVista, and MathVerse.
The primary mechanism evaluation uses 2,000 questions with four fixed
Qwen2.5-VL-7B Base responses each, yielding 8,000 trajectories.
Question selection is frozen before checkpoint evaluation and is blind to
answer correctness, span presence, evidence labels, and mechanism scores.
All four responses are retained, including incorrect answers and responses
without eligible direct claims.

The reference builder joins the frozen question list with the Base
annotation bank by dataset and sample ID. It requires four successfully
annotated, nonempty responses and a resolvable image per question;
incomplete inputs raise an error rather than replacing questions.

\subsection{Scoring and Aggregation}

Mechanism results use the seed-0 checkpoint for each trained model;
temporal comparisons use checkpoints from the seed-0 run.
Question-cluster bootstrap intervals quantify uncertainty over evaluation
questions, not variation across training seeds.

Each checkpoint teacher-forces the same responses and exact-aligned spans
under original and hash-fixed intervened images. Rollouts are averaged
within question, and paired bootstrap intervals resample question clusters
identified by dataset and sample ID. Responses without eligible spans
contribute no claim-level observation. Thus, 8,000 is the reference-bank
size, not the valid-observation count for every subgroup metric.

\subsection{Temporal Subset and Evaluation Scope}

Figure~\ref{fig:training_dynamics} and Figure~\ref{fig:phase_trajectory}
use 500 questions from this bank, retaining all four responses (2,000
trajectories). 
% Dataset-proportional quotas use largest-remainder rounding;
% within each dataset, questions are ranked by SHA256 with seed 20260907.
The selected IDs and source checksum are frozen across checkpoints.

Table~\ref{tab:mechanism_results}, Figure~\ref{fig:aggregate_unsupported_persistence},
and the ablation mechanism metrics in Table~\ref{tab:gate_ablation_mech}
use the full 2,000-question reference. The nine-benchmark accuracy results,
including Table~\ref{tab:gate_ablation_acc}, use free generation instead.

\input{appendices/benchmarks}

%% file: appendices/benchmarks.tex
\subsection{Benchmark Coverage}
\label{app:benchmarks}

Task-level accuracy is evaluated on MMK12, MathVerse, DynaMath,
MathVision, Geometry3K, We-Math, LogicVista, MMMU-Pro, and Clever-Count.
These benchmarks cover visual mathematics, geometry, counting, and
broader multimodal reasoning. Matched models share answer extraction and
verification, with eight sampled responses per question in the reported
mean-accuracy evaluations. Direct EFS and persistence are evaluated on
CrossBench4, rather than on all nine accuracy benchmarks.

\paragraph{HallusionBench evaluation.}
\label{app:hallusion_protocol}
We additionally evaluate the six Qwen2.5-VL-7B variants on
HallusionBench~\citep{guan2024hallusionbench}, using one checkpoint per model
rather than averaging over training seeds. Kimi-K3~\citep{team2026kimi} judges answer
correctness at a judge decoding temperature of 0.6, with one judgment
per response. We report the mean of eight evaluations, computing each
metric within an evaluation before averaging across the eight repeats.
aAcc measures per-answer accuracy; fAcc requires all questions associated
with an image to be correct; qAcc requires a question to be answered
correctly across its visual variants. These results use LLM judging,
unlike the rule-based answer verification in the nine-benchmark suite,
and do not enter its reported average.

The default mechanism intervention uses hash-fixed random-grid blackening
with $14\times14$-pixel patches and blackening probability 0.5. Additional
corruptions are compared in Appendix~\ref{app:intervention_robustness}.

%% file: appendices/routing_audits.tex
\section{Claim Routing and Annotation Audits}
\label{app:router_human_audit}

\subsection{Human Audit}

Two annotators label a stratified sample of 500 questions from
CrossBench4-2000. We evaluate direct visual-claim identification,
direct-versus-derived classification, and supported-versus-unsupported
classification. Table~\ref{tab:router_audit} separates agreement with human
consensus from agreement between the annotators.
\begin{table}[!htbp]
    \centering
    \caption{\textbf{Router and human audit on 500 CrossBench4 questions.}
Precision, recall, F1, and the first three $\kappa$ values use human
consensus as the reference.}
    \label{tab:router_audit}
    \fontsize{8.5}{10}\selectfont
    \setlength{\tabcolsep}{5pt}
    \renewcommand{\arraystretch}{1.18}
    \begin{tabular}{@{}lcccc@{}}
        \toprule
        \textbf{Task} & \textbf{Precision} & \textbf{Recall}
        & \textbf{F1} & \textbf{$\kappa$} \\
        \midrule
        Direct visual claim & 0.90 & 0.87 & 0.88 & 0.81 \\
        Direct vs.\ derived & 0.85 & 0.83 & 0.84 & 0.76 \\
        Supported vs.\ unsupported & 0.82 & 0.78 & 0.80 & 0.71 \\
        \midrule
        \multicolumn{5}{@{}l}{\textit{Human--human Cohen's $\kappa$ (same task order)}} \\
        Direct visual claim & --- & --- & --- & 0.84 \\
        Direct vs.\ derived & --- & --- & --- & 0.79 \\
        Supported vs.\ unsupported & --- & --- & --- & 0.74 \\
        \bottomrule
    \end{tabular}
\end{table}

Agreement is highest for direct-claim identification and lower for
evidence-status classification. PACG uses the former for routing and
reserves supported/unsupported labels for offline analysis.

\subsection{Router Selection}
\label{app:router_selection}

On a separate 200-question expert-annotated set, Qwen3.8-27B gives the
highest F1 (0.86). Qwen3.6-35B-A3B reaches 0.84 F1 with 165 versus
425 seconds per batch, a $2.6\times$ latency advantage; we use it as
the production router. Recall measures matched expert spans, whereas
coverage measures questions with at least one matched expert span.
\begin{table}[!htbp]
    \centering
    \caption{\textbf{Span-extractor comparison on 200 expert-annotated questions.}
Span coverage is the fraction of questions with at least one matched
expert span. Timing is wall-clock seconds per batch on eight A100 GPUs;
shading identifies the selected router.}
    \label{tab:router_selection}
    \fontsize{8.5}{10}\selectfont
    \setlength{\tabcolsep}{5pt}
    \renewcommand{\arraystretch}{1.25}
    \begin{tabular}{@{}lccccc@{}}
        \toprule
        \textbf{Router} & \textbf{Precision} & \textbf{Recall}
        & \textbf{F1} & \shortstack{\textbf{Span}\\\textbf{Coverage}}
        & \textbf{Timing (s)} \\
        \midrule
        Qwen3-4B & 0.68 & 0.63 & 0.65 & 76\% & 108 \\
        Qwen3.8-27B & 0.87 & 0.85 & 0.86 & 94\% & 425 \\
        \rowcolor{oursblue}
        \textbf{Qwen3.6-35B-A3B}
        & \textbf{0.85} & \textbf{0.83} & \textbf{0.84}
        & \textbf{92\%} & \textbf{165} \\
        \bottomrule
    \end{tabular}
\end{table}

\input{appendices/routing_example}

\subsection{Base Claim Annotation Census}
\label{app:crossbench_base_annotation}

The broader CrossBench4 census contains 5,228 questions and 20,912 Base
trajectories, with four responses per question across Geometry3K, MMK12,
LogicVista, and MathVerse. It is separate from the selected 2,000-question
mechanism reference. The annotation client receives the original images,
question, and trace, without the gold answer or source checkpoint identity.
Its request format is reproduced in
Appendix~\ref{app:crossbench_annotation_request}.

We estimate prevalence from the completed span annotations. Among 8,521
answer-correct trajectories, we count each trajectory once if at least
one exact-aligned span has category
\texttt{hallucinated\_or\_unsupported\_direct\_vef}.
This yields 2,370 trajectories (27.81\%); the category includes both
contradicted and unsupported direct visual claims. Exact alignment checks
text and character offsets, not label correctness. The confidence
threshold of 0.9 applies to a separate high-quality audit subset, not to
this prevalence estimate. Exported representative cases are selected
for human inspection.

\paragraph{Relationship to the claim-level unsupported rate.}
The 27.81\% reported in the introduction is a \emph{trajectory-level}
prevalence: it counts each of 8,521 answer-correct Base trajectories once
if \emph{any} of its spans is annotated as
\texttt{hallucinated\_or\_unsupported\_direct\_vef}. The 24.19\% in
Table~\ref{tab:avoidance_audit} is a \emph{claim-level} rate: it is the
fraction of qualifying direct claims that are annotated as unsupported,
averaged over responses with at least one qualifying claim and then over
questions with a defined value. The two quantities also use different populations: the census covers
answer-correct Base trajectories across four benchmarks, whereas the
free-generation audit covers successfully annotated MMK12 responses,
including both correct and incorrect answers, and applies a confidence
threshold of 0.9. They are therefore not expected to coincide. The trajectory-level
prevalence measures how often a rewarded trajectory contains at least one
problematic claim; the claim-level rate measures how often an individual
direct visual claim is unsupported.

\input{appendices/avoidance_audit}

%% file: appendices/routing_example.tex
\subsection{Qualitative Examples of Direct Visual Claim Routing}
\label{app:claim_routing_example}

The three examples below use real pre-RL responses from the MMK12 reference
bank. Each separates a claim's evidential role from its factual support:
both highlighted observations are routed, although only one agrees with
the image. The response still reaches the correct final option.

% Source: phase1/repro_imgfix_20260720/reference_mmk12_500_spans_Qwen_aligned.jsonl
% sample_id: mmk12_2f53458b788d; trace_source: grpo_0step; span IDs: 0 and 2.
% Image: audit pair_048.png, original (left) panel only, header excluded.
% This is a qualitative MMK12 case, not a sample claimed to belong to the
% broader CrossBench4 census. Ellipses explicitly mark omitted trace text.
\begin{tcolorbox}[breakable,
    title={Example 1: A correct answer with a contradicted visual claim},
    colback=blue!2,colframe=pacgpromptblue,colbacktitle=pacgpromptblue,
    coltitle=white,fonttitle=\bfseries,boxrule=0.5pt,arc=1pt,
    left=7pt,right=7pt,top=6pt,bottom=6pt]
\small
\textbf{Question (abridged).} Given the graph of $y=ax^2+bx+c$
($a\ne0$), which conclusion is correct?
Option D states that $ax^2+bx+c=-1$ has two distinct real roots.

\smallskip
\begin{minipage}[c]{0.26\linewidth}
    \centering
    \includegraphics[width=\linewidth,trim=0 0 520 25,clip]
        {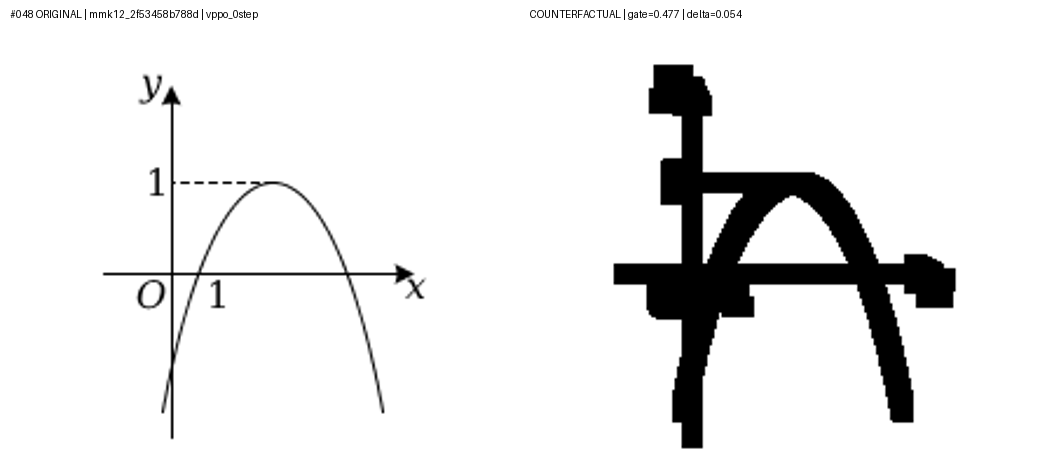}
\end{minipage}\hfill
\begin{minipage}[c]{0.71\linewidth}
\textbf{Original response excerpts:}

\smallskip
\colorbox{blue!12}{\strut The parabola opens downwards}\textcolor{black!65}{,
which means the coefficient $a$ is negative. So, $a < 0$.}

\smallskip
\textcolor{black!55}{[\ldots]}

\smallskip
\colorbox{red!12}{\parbox{\dimexpr\linewidth-2\fboxsep\relax}{%
The y-intercept of the parabola is at $y = c$, which is positive}}
\textcolor{black!65}{So, $c > 0$.}

\smallskip
\textcolor{black!55}{[\ldots]}
Given the above analysis, the correct conclusion is: $\boxed{D}$
\end{minipage}

\medskip
\setlength{\tabcolsep}{3pt}
\renewcommand{\arraystretch}{1.12}
\begin{tabular}{@{}p{0.24\linewidth}p{0.32\linewidth}p{0.37\linewidth}@{}}
\toprule
\rowcolor{blue!10}
\textbf{Excerpt} & \textbf{Evidential role} & \textbf{Offline status}\\
\midrule
\textcolor{blue!65!black}{Blue observation} & Direct visual claim & Supported\\
\textcolor{red!65!black}{Red observation} & Direct visual claim & Contradicted\\
\textcolor{black!65}{Gray deductions} & Derived / non-direct & Not used for routing\\
Final option D & Answer, not a visual claim & Correct\\
\bottomrule
\end{tabular}
\end{tcolorbox}

The graph opens downward and crosses the $y$-axis below zero, contradicting
the red claim. \textbf{Red and blue spans are both routed by PACG; their
support status is shown only for offline analysis.} The gray text draws
conclusions from observations rather than reading new facts from the image.

\smallskip
{\footnotesize Ground truth: D. Ellipses mark omissions; punctuation and math
typesetting are normalized. Sample: \texttt{mmk12\_2f53458b788d}.}

\input{appendices/routing_examples_additional}

%% file: appendices/routing_examples_additional.tex
% Real pre-RL MMK12 excerpts; provenance in figure/routing_examples_additional_source.json.

\medskip
\begin{tcolorbox}[breakable,
 title={Example 2: Extrema are not zeros},
 colback=blue!2,colframe=pacgpromptblue,colbacktitle=pacgpromptblue,
 coltitle=white,fonttitle=\bfseries,boxrule=0.5pt,arc=1pt,
 left=7pt,right=7pt,top=6pt,bottom=6pt]
\small
\textbf{Question (abridged).} From the plotted cubic $f(x)=x^3+bx^2+cx+d$, determine where
$y=\log_2(x^2+\frac23 bx+\frac c3)$ is decreasing.

\smallskip
\begin{minipage}[c]{0.26\linewidth}
\centering
\includegraphics[width=\linewidth,trim=0 0 520 25,clip]{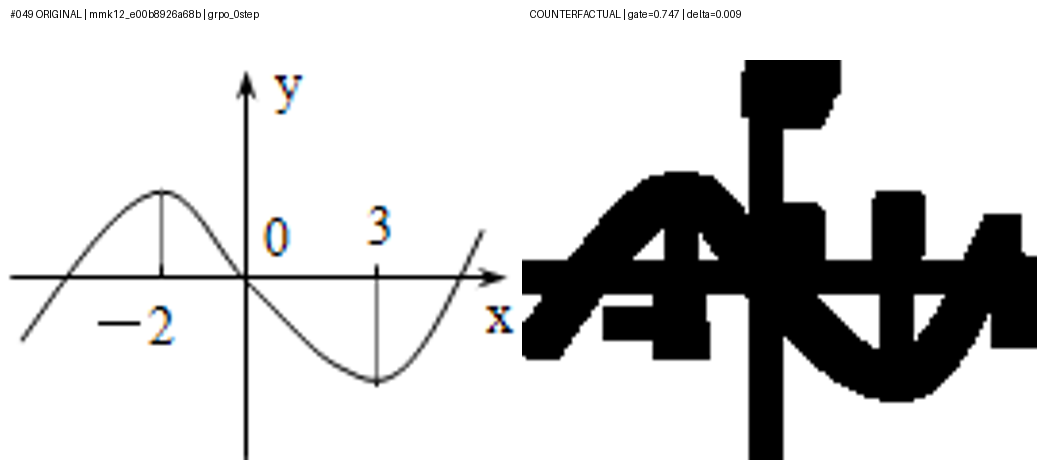}
\end{minipage}\hfill
\begin{minipage}[c]{0.71\linewidth}
\textbf{Original response excerpts:}

\smallskip
\colorbox{blue!12}{\parbox{\dimexpr\linewidth-2\fboxsep\relax}{The function has a local maximum at $x=-2$ and a local minimum at $x=3$.}}

\smallskip
\colorbox{red!12}{\parbox{\dimexpr\linewidth-2\fboxsep-2pt\relax}{The function crosses the x-axis at $x=-2$ and $x=3$}}, \textcolor{black!65}{so $-2$ and $3$ are roots of the polynomial.}

\smallskip
\textcolor{black!55}{[\ldots]} $\boxed{A}$
\end{minipage}

\medskip
\setlength{\tabcolsep}{3pt}
\renewcommand{\arraystretch}{1.12}
\begin{tabular}{@{}p{0.24\linewidth}p{0.32\linewidth}p{0.37\linewidth}@{}}
\toprule
\rowcolor{blue!10}
\textbf{Excerpt} & \textbf{Evidential role} & \textbf{Offline status}\\
\midrule
\textcolor{blue!65!black}{Blue observation} & Direct visual claim & Supported\\
\textcolor{red!65!black}{Red observation} & Direct visual claim & Contradicted\\
\textcolor{black!65}{Gray reasoning} & Derived / non-direct & Not used for routing\\
Final option A & Answer, not a visual claim & Correct\\
\bottomrule
\end{tabular}
\end{tcolorbox}
The marked locations are turning points above and below the horizontal axis, not crossings of that axis.
Both colored claims are routed; support labels remain offline annotations.

\smallskip
{\footnotesize Ground truth: A. Ellipses mark omissions; math typesetting
is normalized. Sample: \texttt{mmk12\_e00b8926a68b}.}

\medskip
\begin{tcolorbox}[breakable,
 title={Example 3: Reading a circuit versus inferring a voltage},
 colback=blue!2,colframe=pacgpromptblue,colbacktitle=pacgpromptblue,
 coltitle=white,fonttitle=\bfseries,boxrule=0.5pt,arc=1pt,
 left=7pt,right=7pt,top=6pt,bottom=6pt]
\small
\textbf{Question (abridged).} With the switch closed, $V_1=7.5\,\mathrm V$,
$V_2=9\,\mathrm V$, and supply voltage $12\,\mathrm V$, find the voltage across $L_2$.

\smallskip
\begin{minipage}[c]{0.26\linewidth}
\centering
\includegraphics[width=\linewidth,trim=0 0 520 25,clip]{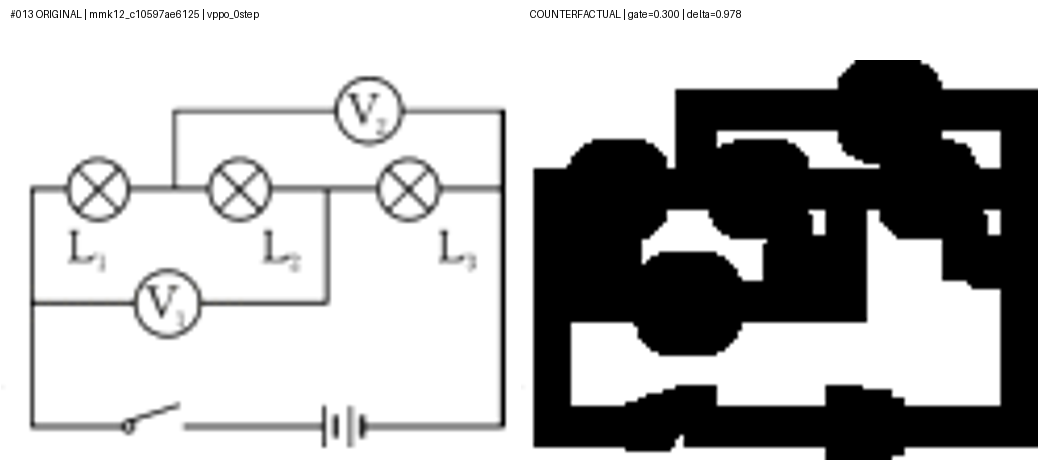}
\end{minipage}\hfill
\begin{minipage}[c]{0.71\linewidth}
\textbf{Original response excerpts:}

\smallskip
\colorbox{red!12}{\parbox{\dimexpr\linewidth-2\fboxsep\relax}{The circuit is in a parallel connection.}}

\smallskip
\colorbox{blue!12}{\parbox{\dimexpr\linewidth-2\fboxsep\relax}{$V_1$ measures the voltage across $L_1$ and $L_2$.}}

\smallskip
\textcolor{black!55}{[\ldots]} \textcolor{black!65}{Therefore, the voltage across $L_2$ is $4.5\,\mathrm V$.}

\smallskip
\textcolor{black!55}{[\ldots]} $\boxed{A}$
\end{minipage}

\medskip
\setlength{\tabcolsep}{3pt}
\renewcommand{\arraystretch}{1.12}
\begin{tabular}{@{}p{0.24\linewidth}p{0.32\linewidth}p{0.37\linewidth}@{}}
\toprule
\rowcolor{blue!10}
\textbf{Excerpt} & \textbf{Evidential role} & \textbf{Offline status}\\
\midrule
\textcolor{blue!65!black}{Blue observation} & Direct visual claim & Supported\\
\textcolor{red!65!black}{Red observation} & Direct visual claim & Contradicted\\
\textcolor{black!65}{Gray reasoning} & Derived / non-direct & Not used for routing\\
Final option A & Answer, not a visual claim & Correct\\
\bottomrule
\end{tabular}
\end{tcolorbox}
The lamps form a series path. The supported span identifies the two lamps bridged by the voltmeter; the later voltage statement is a computed conclusion.
Both colored claims are routed; support labels remain offline annotations.

\smallskip
{\footnotesize Ground truth: A. Ellipses mark omissions; math typesetting
is normalized. Sample: \texttt{mmk12\_c10597ae6125}.}

%% file: appendices/avoidance_audit.tex
\subsection{Free-Generation Avoidance Audit}
\label{app:avoidance_audit}

We audit each Qwen2.5-VL-7B checkpoint's own free generations on the
2,000-question MMK12 test set, with eight rollouts per question (16,000
expected responses per model). The audit examines whether changes in
unsupported visual claims accompany a broad reduction in visual content,
shorter answers, or substitution of uncertainty and abstention for claims.
All annotation-based metrics, including accuracy, use successfully annotated
responses; failed annotations are excluded rather than counted as incorrect.
Available rollouts are first macro-averaged within each question. Bracketed
intervals in Table~\ref{tab:avoidance_audit} are 95\% confidence intervals
obtained by bootstrapping question clusters.

We annotate these free generations with Qwen3.8-max using the multimodal
span-extraction request documented in Appendix~\ref{app:avoidance_annotation_prompt}
(Prompt~3), supplying the original images, question, and generated response
while hiding the source checkpoint's identity.

Direct Claim Rate is the fraction of responses containing at least one
qualifying direct visual claim. Direct Claims / Response counts those claims,
and Direct Token Ratio measures the fraction of response tokens covered by
their union. Claims must satisfy the direct-visual routing criteria, exact
text alignment, and annotation confidence of at least 0.9. Unsupported Claim
Rate is the fraction of qualifying direct claims annotated as unsupported,
averaged over responses with at least one qualifying claim and then over
questions with a defined value. Unsupported Response Rate is the fraction of
responses containing any such claim. Response length is measured in
assistant tokens.

\begin{table*}[!htbp]
    \centering
    \caption{\textbf{Free-generation visual-claim audit on MMK12.}
Question-macro means with 95\% question-cluster bootstrap intervals,
using successfully annotated responses. Rates and accuracy are percentages;
lengths are tokens. Uncertainty/abstention entries are point values or
ranges, not confidence intervals.}
    \label{tab:avoidance_audit}
    \begingroup
    \footnotesize
    \setlength{\tabcolsep}{4pt}
    \renewcommand{\arraystretch}{1.25}
    \resizebox{\linewidth}{!}{%
    \begin{tabular}{@{}lcccccccc@{}}
        \toprule
        \textbf{Model}
        & \shortstack{\textbf{Direct Claim}\\\textbf{Rate (\%)}}
        & \shortstack{\textbf{Direct Claims}\\\textbf{/ Response}}
        & \shortstack{\textbf{Direct Token}\\\textbf{Ratio (\%)}}
        & \shortstack{\textbf{Response}\\\textbf{Length}}
        & \shortstack{\textbf{Unsup. Claim}\\\textbf{Rate (\%)} $\downarrow$}
        & \shortstack{\textbf{Unsup. Resp.}\\\textbf{Rate (\%)} $\downarrow$}
        & \shortstack{\textbf{Accuracy}\\\textbf{(\%)}}
        & \shortstack{\textbf{Uncert./}\\\textbf{Abst. (\%)}} \\
        \midrule
        Base
        & \shortstack{77.46\\{[76.08, 78.81]}}
        & \shortstack{2.463\\{[2.387, 2.537]}}
        & \shortstack{11.18\\{[10.83, 11.54]}}
        & \shortstack{506.62\\{[498.61, 514.84]}}
        & \shortstack{24.19\\{[23.14, 25.27]}}
        & \shortstack{32.19\\{[30.88, 33.48]}}
        & \shortstack{43.02\\{[41.56, 44.51]}}
        & 10.2 \\
        GRPO
        & \shortstack{78.05\\{[76.60, 79.48]}}
        & \shortstack{2.382\\{[2.306, 2.459]}}
        & \shortstack{11.14\\{[10.75, 11.52]}}
        & \shortstack{526.64\\{[517.18, 535.97]}}
        & \shortstack{19.61\\{[18.53, 20.73]}}
        & \shortstack{26.61\\{[25.22, 27.99]}}
        & \shortstack{73.41\\{[71.76, 75.07]}}
        & 9.5--10.5 \\
        DAPO
        & \shortstack{76.12\\{[74.66, 77.57]}}
        & \shortstack{2.148\\{[1.912, 2.314]}}
        & \shortstack{10.30\\{[9.94, 10.66]}}
        & \shortstack{516.82\\{[508.45, 525.13]}}
        & \shortstack{18.84\\{[17.79, 19.88]}}
        & \shortstack{25.83\\{[24.48, 27.17]}}
        & \shortstack{81.60\\{[79.87, 83.28]}}
        & 9.5--10.5 \\
        DAPO + PACG
        & \shortstack{74.52\\{[73.07, 75.96]}}
        & \shortstack{2.211\\{[2.139, 2.282]}}
        & \shortstack{9.87\\{[9.54, 10.21]}}
        & \shortstack{528.46\\{[520.12, 536.96]}}
        & \shortstack{16.16\\{[14.11, 18.25]}}
        & \shortstack{25.62\\{[24.55, 27.10]}}
        & \shortstack{83.16\\{[81.56, 84.75]}}
        & 9.0--11.0 \\
        VPPO
        & \shortstack{79.41\\{[78.08, 80.75]}}
        & \shortstack{2.645\\{[2.562, 2.733]}}
        & \shortstack{9.95\\{[9.62, 10.28]}}
        & \shortstack{681.46\\{[670.84, 692.14]}}
        & \shortstack{17.44\\{[16.51, 18.41]}}
        & \shortstack{25.61\\{[24.33, 26.87]}}
        & \shortstack{81.73\\{[80.16, 83.29]}}
        & 9.5--10.5 \\
        VPPO + PACG
        & \shortstack{77.90\\{[76.46, 79.28]}}
        & \shortstack{2.463\\{[2.383, 2.543]}}
        & \shortstack{9.31\\{[8.99, 9.62]}}
        & \shortstack{697.47\\{[686.76, 708.45]}}
        & \shortstack{16.61\\{[15.69, 17.55]}}
        & \shortstack{24.12\\{[22.90, 25.37]}}
        & \shortstack{82.63\\{[81.06, 84.14]}}
        & 9.0--11.0 \\
        \bottomrule
    \end{tabular}}
    \endgroup
\end{table*}

\begin{table*}[!htbp]
    \centering
    \caption{\textbf{Response-length medians and annotation coverage.} Coverage is out of 16,000 expected
    responses per model. Claim counts are exact high-confidence direct claims.}
    \label{tab:avoidance_coverage}
    \small
    \setlength{\tabcolsep}{5pt}
    \renewcommand{\arraystretch}{1.2}
    \begin{tabular}{@{}lrrr@{}}
        \toprule
        Model & \shortstack{Median length\\(tokens)}
        & \shortstack{Annotated\\responses}
        & \shortstack{Direct\\claims} \\
        \midrule
        Base & 462.0 & 16,000/16,000 & 39,401 \\
        GRPO & 474.0 & 15,996/16,000 & 38,098 \\
        DAPO & 480.0 & 15,999/16,000 & 34,360 \\
        DAPO + PACG & 489.0 & 15,998/16,000 & 35,370 \\
        VPPO & 635.0 & 15,990/16,000 & 42,295 \\
        VPPO + PACG & 645.0 & 15,992/16,000 & 39,375 \\
        \bottomrule
    \end{tabular}
\end{table*}

Relative to DAPO, PACG lowers direct-claim rate by 1.60 percentage
points and unsupported-claim rate by 2.68 points; mean response length
increases from 516.82 to 528.46 tokens. With VPPO, the corresponding
rate reductions are 1.51 and 0.83 points, while mean length rises from
681.46 to 697.47 tokens. Median lengths also increase
(Table~\ref{tab:avoidance_coverage}). Thus, the lower unsupported-claim
rates are accompanied by continued visual-claim generation rather than
shorter responses. The reported uncertainty ranges remain descriptive.

%% file: appendices/mechanism_analysis.tex
\section{Supplementary Mechanism Analysis}
\label{app:supp_mechanism}

\subsection{Image-versus-Text Difference-in-Differences}
\label{app:image_text_did}
Let $\overline{\widetilde{\Delta}}^{\,I}_{M}$ and
$\overline{\widetilde{\Delta}}^{\,T}_{M}$ denote the mean corrected
persistence of model $M$ on direct visual claims under image and text
interventions, respectively. The image intervention holds the question
text and response fixed; the text intervention holds the original image
and response fixed while perturbing the question text. We report
\begin{equation}
\operatorname{DID}(M)=
\left(\overline{\widetilde{\Delta}}^{\,I}_{M}
-\overline{\widetilde{\Delta}}^{\,I}_{\mathrm{Base}}\right)
-\left(\overline{\widetilde{\Delta}}^{\,T}_{M}
-\overline{\widetilde{\Delta}}^{\,T}_{\mathrm{Base}}\right).
\end{equation}
Image and text measurements are paired by sample before aggregation.
Base therefore has DID zero; comparisons between PACG and its parent
subtract their Base-relative DID values. Lower values indicate a more
negative image-specific persistence change relative to the text control.
The accuracy column in Table~\ref{tab:mechanism_results} is evaluated
by free generation on the CrossBench4-2000 questions; it is separate
from teacher-forced scoring of the fixed Base responses.

\subsection{Supported-Claim Persistence and Bootstrap Intervals}
\label{app:supported_persistence}

Table~\ref{tab:supported_persistence} complements the main mechanism
table with supported-claim corrected persistence on the same fixed
reference. Lower values indicate greater relative retractability.
\begin{table}[!htbp]
    \centering
    \caption{\textbf{Supported-claim corrected persistence.}
    Lower is better. \textbf{Bold} and \underline{underlining} mark the best
    and second-best point estimates; blue shading identifies PACG variants.}
    \label{tab:supported_persistence}
    \setlength{\tabcolsep}{8pt}
    \renewcommand{\arraystretch}{1.15}
    \begin{tabular}{@{}lc@{}}
        \toprule
        \textbf{Model} & \textbf{Supported Persistence} $\downarrow$ \\
        \midrule
        Base & $-0.137$ \\
        GRPO & $\mathbf{-0.318}$ \\
        DAPO & $-0.159$ \\
        VPPO & $-0.184$ \\
        \midrule
        \rowcolor{oursblue}
        \textbf{DAPO+\pacg{}} & $-0.165$ \\
        \rowcolor{oursbluestrong}
        \textbf{VPPO+\pacg{}} & \underline{$-0.213$} \\
        \bottomrule
    \end{tabular}
\end{table}

Table~\ref{tab:mechanism_effect_cis} reports paired question-clustered
intervals for the principal effects. Positive changes in unsupported
persistence indicate weaker retractability. For the subgroup interaction
$\Delta_{\mathrm{unsup}}-\Delta_{\mathrm{sup}}$, positive values mean
greater deterioration on unsupported claims; negative PACG interactions
mean greater improvement on that subgroup.
\begin{table*}[!htbp]
    \centering
    \caption{\textbf{Paired mechanism effects with question-clustered
95\% bootstrap intervals.} Each effect is the first model minus the second.
Accuracy effects are on the $[0,1]$ scale; other effects use their metric's
native scale.}
    \label{tab:mechanism_effect_cis}
    \fontsize{8.2}{9.6}\selectfont
    \setlength{\tabcolsep}{5pt}
    \renewcommand{\arraystretch}{1.18}
    \begin{tabular}{@{}llcc@{}}
        \toprule
        \textbf{Comparison} & \textbf{Metric}
        & \textbf{Effect} & \textbf{95\% CI} \\
        \midrule
        DAPO $-$ Base & Unsupported persistence
        & $+0.0651$ & $[0.022,\,0.108]$ \\
        VPPO $-$ Base & Unsupported persistence
        & $+0.0831$ & $[0.036,\,0.130]$ \\
        DAPO+\pacg{} $-$ DAPO & Unsupported persistence
        & $-0.0151$ & $[-0.022,\,-0.008]$ \\
        VPPO+\pacg{} $-$ VPPO & Unsupported persistence
        & $-0.0819$ & $[-0.115,\,-0.049]$ \\
        \midrule
        DAPO $-$ Base & Direct EFS
        & $+0.0715$ & $[0.025,\,0.118]$ \\
        VPPO $-$ Base & Direct EFS
        & $+0.2330$ & $[0.160,\,0.306]$ \\
        DAPO+\pacg{} $-$ DAPO & Direct EFS
        & $+0.0290$ & $[0.006,\,0.052]$ \\
        VPPO+\pacg{} $-$ VPPO & Direct EFS
        & $+0.0369$ & $[0.010,\,0.064]$ \\
        \midrule
        DAPO $-$ Base & Overall persistence
        & $-0.0803$ & $[-0.122,\,-0.039]$ \\
        VPPO $-$ Base & Overall persistence
        & $-0.0791$ & $[-0.121,\,-0.037]$ \\
        DAPO+\pacg{} $-$ DAPO & Overall persistence
        & $-0.0463$ & $[-0.078,\,-0.015]$ \\
        VPPO+\pacg{} $-$ VPPO & Overall persistence
        & $-0.0713$ & $[-0.110,\,-0.033]$ \\
        \midrule
        DAPO $-$ Base & Unsup.--supp. interaction
        & $+0.0872$ & $[0.032,\,0.142]$ \\
        VPPO $-$ Base & Unsup.--supp. interaction
        & $+0.1300$ & $[0.065,\,0.195]$ \\
        DAPO+\pacg{} $-$ DAPO & Unsup.--supp. interaction
        & $-0.0096$ & $[-0.017,\,-0.002]$ \\
        VPPO+\pacg{} $-$ VPPO & Unsup.--supp. interaction
        & $-0.0524$ & $[-0.086,\,-0.019]$ \\
        \midrule
        DAPO+\pacg{} $-$ DAPO & Image-vs-text DID
        & $+0.0059$ & $[0.001,\,0.011]$ \\
        VPPO+\pacg{} $-$ VPPO & Image-vs-text DID
        & $-0.0246$ & $[-0.045,\,-0.004]$ \\
        \midrule
        GRPO $-$ Base & Unsupported persistence
        & $-0.0071$ & $[-0.025,\,0.011]$ \\
        GRPO $-$ Base & Direct EFS
        & $+0.2796$ & $[0.200,\,0.359]$ \\
        DAPO+\pacg{} $-$ DAPO & Accuracy
        & $+0.0079$ & $[0.001,\,0.015]$ \\
        VPPO+\pacg{} $-$ VPPO & Accuracy
        & $+0.0082$ & $[0.001,\,0.016]$ \\
        \bottomrule
    \end{tabular}
\end{table*}

\input{appendices/optimizer_credit_exposure}

\subsection{Checkpoint Training Dynamics and Sensitivity--Retractability Phase Trajectory}
\label{app:phase_trajectory}

\paragraph{Training dynamics reproduce the endpoint pattern.}
Figure~\ref{fig:training_dynamics} tracks DAPO and DAPO+\pacg{} from step 0
to 200 on the fixed 500-question subset of CrossBench4-2000. Under
DAPO, EFS increases from 0.2189 to 0.2856 while unsupported persistence
moves from $-0.1369$ to $-0.1227$, showing weaker retraction despite
greater visual sensitivity. PACG follows a different trajectory: EFS reaches
0.3117 while unsupported persistence reaches $-0.1441$. The checkpoint
dynamics therefore mirror the endpoint result and show that PACG prevents
the unsupported-persistence drift observed during DAPO training.

\begin{figure}[!htbp]
    \centering
    \includegraphics[width=\linewidth]{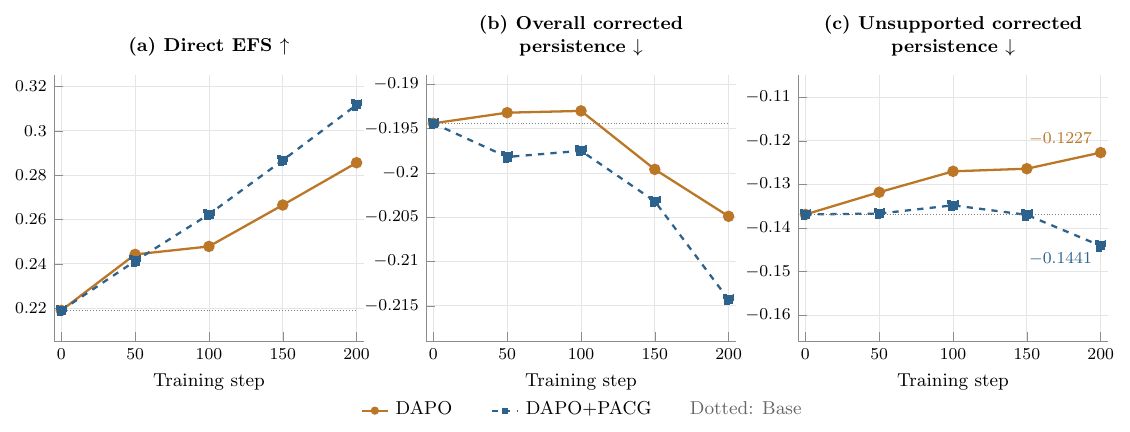}
    \caption{\textbf{Sensitivity and unsupported-claim persistence diverge
    during training.}
    DAPO and DAPO+\pacg{} are evaluated at matched steps
    $0,50,100,150,$ and $200$.
    Direct EFS increases for both methods, whereas DAPO's unsupported
    persistence becomes progressively less negative. PACG prevents the same
    drift and ends with greater unsupported retractability than DAPO.
    The temporal series uses 500 fixed questions sampled from CrossBench4-2000,
    with four Base rollouts each; subset averages can differ from full-set
    averages. Lines connect
    observed checkpoints rather than fitted curves.}
    \label{fig:training_dynamics}
\end{figure}

\paragraph{Sensitivity--retractability phase trajectory.}
Figure~\ref{fig:phase_trajectory} combines the direct EFS and unsupported
persistence measurements from Figure~\ref{fig:training_dynamics}.
DAPO ends below and to the right of Base, whereas DAPO+\pacg{} ends above
and to the right. Intermediate PACG checkpoints fluctuate around the
initial retractability level rather than improving monotonically.

\subsection{Robustness to Counterfactual Corruption}
\label{app:intervention_robustness}

We compare three corruptions using the same fixed reference and checkpoint
convention. Table~\ref{tab:intervention_robustness} reports absolute PACG
measurements and matched changes from DAPO. Valid rate is the fraction of
responses for which counterfactual scoring completes.
\begin{table}[htbp]
    \centering
    \caption{\textbf{Robustness to counterfactual corruption.}
Absolute DAPO+\pacg{} scores and the matched DAPO+\pacg{} minus DAPO
change in unsupported persistence (PACG Effect). Entries are point estimates.}
    \label{tab:intervention_robustness}
    \fontsize{8.5}{10}\selectfont
    \setlength{\tabcolsep}{5pt}
    \renewcommand{\arraystretch}{1.25}
    \begin{tabular}{@{}lccccc@{}}
        \toprule
        \textbf{Intervention}
        & \shortstack{\textbf{EFS}\\$\uparrow$}
        & \shortstack{\textbf{Overall}\\\textbf{Pers.} $\downarrow$}
        & \shortstack{\textbf{Unsup.}\\\textbf{Pers.} $\downarrow$}
        & \shortstack{\textbf{PACG Effect}\\$\downarrow$}
        & \shortstack{\textbf{Valid}\\\textbf{Rate}} \\
        \midrule
        Random mask       & 0.31 & $-0.2558$ & $-0.139$ & $-0.015$ & 96\% \\
        Coarse ink        & 0.28 & $-0.148$ & $-0.125$ & $-0.013$ & 94\% \\
        Foreground/salient & 0.33 & $-0.172$ & $-0.151$ & $-0.018$ & 93\% \\
        \bottomrule
    \end{tabular}
\end{table}

PACG lowers unsupported persistence under all three corruptions, with
effects from $-0.013$ to $-0.018$. This consistency supports robustness
of the PACG improvement to the tested corruption operators.

\paragraph{How the images are intervened on.}
Random mask partitions the image into $14\times14$-pixel grid cells and
independently blackens each cell with probability 0.5; a sample-specific
hash fixes the random pattern. Coarse ink thresholds the grayscale image
at 245 (on the 0--255 scale), expands the resulting dark-pixel mask with
a $9\times9$ maximum filter, and paints the expanded mask black.
Foreground/salient uses the bounding rectangle of pixels below the same
threshold, pads it on each side by 5\% of the larger image dimension
(clipped to image boundaries), and blackens the entire rectangle.
Thus, ``salient'' here refers to a coarse foreground heuristic, not a
learned saliency map or a claim-specific evidence region.

\begin{figure}[!htbp]
\centering
\begin{subfigure}{0.24\linewidth}
\includegraphics[width=\linewidth]{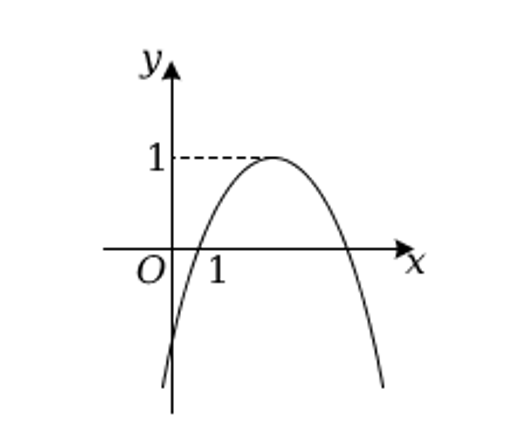}
\caption{Original}
\end{subfigure}\hfill
\begin{subfigure}{0.24\linewidth}
\includegraphics[width=\linewidth]{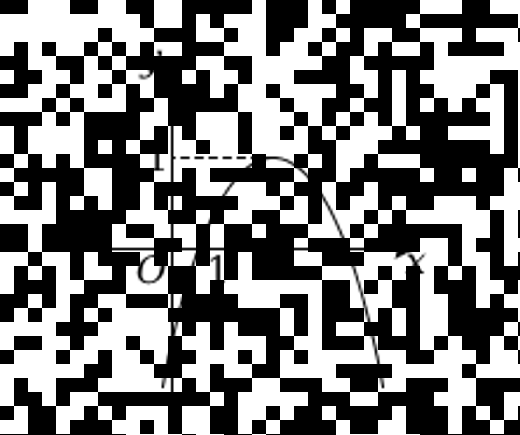}
\caption{Random mask}
\end{subfigure}\hfill
\begin{subfigure}{0.24\linewidth}
\includegraphics[width=\linewidth]{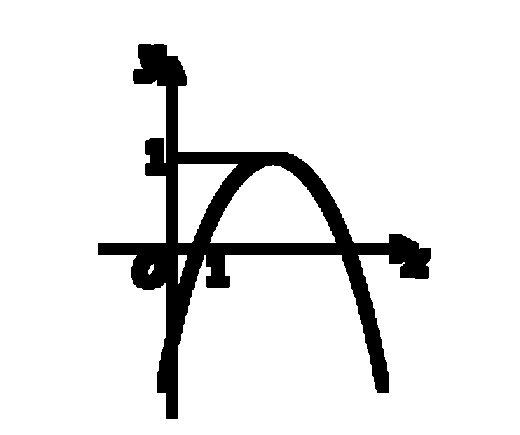}
\caption{Coarse ink}
\end{subfigure}\hfill
\begin{subfigure}{0.24\linewidth}
\includegraphics[width=\linewidth]{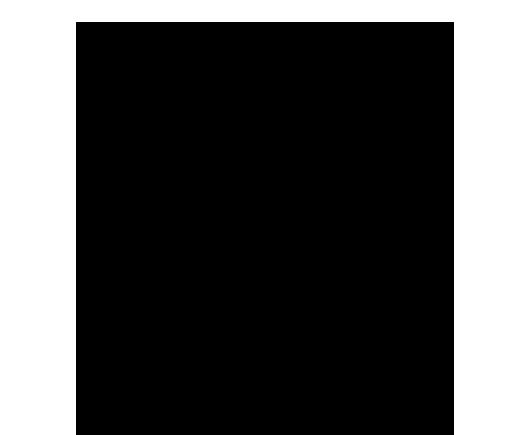}
\caption{Foreground box}
\end{subfigure}
\caption{\textbf{Three image interventions on the same example.}
Illustrative outputs of the evaluation operators using their default
parameters on the original panel recovered from a saved MMK12 audit image.
Random mask removes grid cells; coarse ink expands dark strokes;
foreground masking removes their padded bounding rectangle. These are
image-level controls, not verified claim-localized evidence removals.
The illustration does not add observations to
Table~\ref{tab:intervention_robustness}.}
\label{fig:intervention_examples}
\end{figure}

%% file: appendices/optimizer_credit_exposure.tex
\subsection{Online Positive-Credit Exposure}
\label{app:optimizer_credit_exposure}

We compare the original GRPO, DAPO, and VPPO parent runs over training
steps 181--200. Unlike the fixed-reference mechanism evaluation, this
analysis concerns online rollouts used for actor updates. The frozen
PACG router identifies direct visual claims, which are aligned to response
tokens. The population includes all routed direct claims, without
supported/unsupported filtering.

We report two complementary metrics. The positive routed-token rate,
$P_{\mathrm{pos}}$, measures the fraction of routed direct-claim tokens
receiving positive parent credit. Effective positive update mass,
$M_{\mathrm{eff}}$, measures the positive update mass applied to those
tokens. Parent credit is measured before PACG and includes VPPO's own
credit shaping. Metrics are aggregated from tokens to rollouts, then to
steps, with equal weighting across the 20 steps.

\begin{table}[!htbp]
    \centering
    \small
    \caption{\textbf{Online positive-credit exposure over steps 181--200.}
    Results concern routed direct visual claims in the parent training
    runs. Values are point estimates.}
    \label{tab:optimizer_credit_exposure}
    \setlength{\tabcolsep}{8pt}
    \begin{tabular}{@{}lccc@{}}
        \toprule
        Metric & GRPO & DAPO & VPPO \\
        \midrule
        Positive routed-token rate, $P_{\mathrm{pos}}$ (\%) & 71.8 & 66.3 & 78.9 \\
        Effective positive update mass, $M_{\mathrm{eff}}$ & 0.35 & 0.93 & 1.00 \\
        \bottomrule
    \end{tabular}
\end{table}

Table~\ref{tab:optimizer_credit_exposure} separates the frequency of
positive credit from its effective magnitude. Positive routed-token rates
are broadly comparable across the three runs, but DAPO and VPPO exhibit
substantially greater effective positive update mass than GRPO. In
particular, DAPO has a lower positive-token rate than GRPO yet a higher
update mass. This pattern is consistent with optimizer-dependent exposure
to positive updates, rather than the presence of credit broadcasting alone,
being relevant to the observed persistence differences. The comparison
is descriptive and does not establish that positive-credit exposure
causes SPD.

%% file: appendices/task_results.tex
\section{Additional Task-Level Results}
\label{app:additional_task_results}
\label{app:further_results}

\subsection{Compatibility with Plain GRPO}
\label{app:grpo_pacg}

We compose PACG with plain GRPO on the same nine benchmarks as
Table~\ref{tab:main_results}. The GRPO row reproduces that table.
\begin{table*}[!htbp]
    \centering
    \caption{\textbf{PACG composed with plain GRPO.}
Accuracy (\%); GRPO+\pacg{} uses mean accuracy@8. Bold marks the higher
displayed value in each column. Entries are point estimates.}
    \label{tab:grpo_pacg}
    \fontsize{7.3}{8.8}\selectfont
    \setlength{\tabcolsep}{2.3pt}
    \renewcommand{\arraystretch}{1.3}
    \resizebox{\linewidth}{!}{%
    \begin{tabular}{@{}l*{10}{c}@{}}
        \toprule
        \textbf{Method}
        & \textbf{MMK12} & \shortstack{\textbf{Math}\\\textbf{Verse}}
        & \shortstack{\textbf{Dyna}\\\textbf{Math}}
        & \shortstack{\textbf{Math}\\\textbf{Vision}}
        & \shortstack{\textbf{Geometry}\\\textbf{3K}}
        & \textbf{We-Math} & \shortstack{\textbf{Logic}\\\textbf{Vista}}
        & \shortstack{\textbf{MMMU-}\\\textbf{Pro}}
        & \shortstack{\textbf{Clever-}\\\textbf{Count}}
        & \textbf{Avg.} \\
        \midrule
        GRPO
        & 72.2 & 67.8 & \textbf{65.1} & 30.1 & 42.0 & 67.8
        & 47.1 & 38.3 & 81.0 & 56.8 \\
        \rowcolor{oursblue}
        \textbf{GRPO+\pacg{}}
        & \textbf{73.5} & \textbf{68.1} & 65.0 & \textbf{30.2}
        & \textbf{42.3} & \textbf{68.8} & \textbf{47.3}
        & \textbf{39.0} & \textbf{81.7} & \textbf{57.3} \\
        \bottomrule
    \end{tabular}}
\end{table*}

GRPO+\pacg{} raises the reported average from 56.8 to 57.3. The largest
displayed gains occur on MMK12 ($+1.3$ percentage points), We-Math ($+1.0$), MMMU-Pro
($+0.7$), and Clever-Count ($+0.7$); DynaMath changes from 65.1 to 65.0.
These results extend the task-level evaluation beyond DAPO and VPPO.
The full nine-benchmark component ablation is reported in
Table~\ref{tab:gate_ablation_acc}.

\subsection{Additional Component Mechanism Metrics}
\label{app:full_ablation_results}
The complete nine-benchmark accuracy breakdown and unsupported persistence
are reported in Table~\ref{tab:gate_ablation}. The table below adds direct
EFS and overall corrected persistence under the same reference protocol.
Variant definitions follow Table~\ref{tab:gate_ablation}.
\begin{table}[!htbp]
\centering
\caption{\textbf{Additional component mechanism metrics.}
Point estimates; lower persistence indicates stronger retraction.
$^\dagger$ denotes a degenerate EFS value associated with training collapse,
not improved visual grounding.}
\label{tab:gate_ablation_mech}
\small
\setlength{\tabcolsep}{7pt}
\begin{tabular}{@{}lccc@{}}
\toprule
Variant & Direct EFS $\uparrow$ & Corr. Pers. $\downarrow$ & Unsup. Pers. $\downarrow$\\
\midrule
DAPO & 0.2806 & $-0.2095$ & $-0.1241$\\
w/o Residual Corr. & 0.4099 & $-0.2410$ & $-0.1278$\\
w/o Direct Routing & $3.0366^\dagger$ & $-0.1766$ & $-0.0374$\\
w/o Positive-Only & 0.2046 & $-0.2224$ & $-0.1301$\\
\rowcolor{oursblue}
PACG (full) & 0.3096 & $\mathbf{-0.2558}$ & $\mathbf{-0.1392}$\\
\bottomrule
\end{tabular}
\end{table}

\subsection{Variation Across Training Seeds}
\label{app:seed_variance}

Table~\ref{tab:seed_variance} reports the nine-benchmark average across
three training seeds for each matched Qwen2.5-VL-7B optimizer.
Both improvements exceed the reported seed-level standard deviations and
are significant under two-sided Welch's $t$-tests ($p<0.05$), providing
evidence that the gains in Table~\ref{tab:main_results} extend beyond the
observed training-seed variation.

\begin{table}[!htbp]
\centering
\caption{\textbf{Seed variance} (Qwen2.5-VL-7B; nine-benchmark Avg., \%;
mean $\pm$ sample standard deviation over three training seeds).
$p$: two-sided Welch's $t$-test between each \pacg{} composition and its
parent ($n=3$ per group); $\Delta$ is in percentage points.}
\label{tab:seed_variance}
\small
\setlength{\tabcolsep}{6pt}
\renewcommand{\arraystretch}{1.15}
\begin{tabular}{@{}lccc@{}}
\toprule
Model & Avg. (\%) & $\Delta$ vs.\ parent & $p$ \\
\midrule
DAPO & $58.1 \pm 0.40$ & -- & -- \\
DAPO+\pacg{} & $59.9 \pm 0.52$ & $+1.8$ & $0.010$ \\
VPPO & $59.8 \pm 0.40$ & -- & -- \\
VPPO+\pacg{} & $60.9 \pm 0.47$ & $+1.1$ & $0.038$ \\
\bottomrule
\end{tabular}
\end{table}

\input{appendices/further_results}

%% file: appendices/further_results.tex
\subsection{Sensitivity to the Persistence Margin}
\label{app:margin_sensitivity}

We vary $m\in\{0.00,0.05,0.10\}$ while fixing $\alpha=10$ and
$g_{\min}=0.30$. Among these settings, $m=0.05$ gives the highest
MMK12 accuracy (81.2\%), an EFS of 0.31, and retains
81\% of positive credit under the reported retention metric.
\begin{table*}[!htbp]
    \centering
    \caption{\textbf{Sensitivity to the persistence margin.}
Gate strength $\alpha=10$ and floor $g_{\min}=0.30$ are fixed.
Accuracy and retention are percentages. Shading marks the selected
operating point; bold marks the best accuracy, EFS, and persistence values.
Entries are point estimates.}
    \label{tab:margin_sensitivity}
    \setlength{\tabcolsep}{5pt}
    \renewcommand{\arraystretch}{1.25}
    \resizebox{\linewidth}{!}{%
    \begin{tabular}{@{}cccccccc@{}}
        \toprule
        $m$ & $\alpha$ & $g_{\min}$
        & \shortstack{MMK12\\Accuracy (\%) $\uparrow$}
        & EFS $\uparrow$
        & \shortstack{Corrected\\Persistence $\downarrow$}
        & \shortstack{Unsupported\\Persistence $\downarrow$}
        & \shortstack{Positive-credit\\retention (\%)} \\
        \midrule
        0.00 & 10 & 0.30 & 81.0 & \textbf{0.34} & $-0.10$ & $-0.08$ & 95 \\
        \rowcolor{oursblue}
        \textbf{0.05} & 10 & 0.30 & \textbf{81.2} & 0.31
        & $-0.23$ & $-0.14$ & 81 \\
        0.10 & 10 & 0.30 & 81.0 & 0.28
        & $\mathbf{-0.28}$ & $\mathbf{-0.20}$ & 65 \\
        \bottomrule
    \end{tabular}}
\end{table*}

The zero-margin setting matches ungated DAPO's MMK12 accuracy (81.0\%)
while retaining 95\% of positive credit.
At $m=0.10$, retention falls to 65\%, EFS to 0.28, and accuracy to
81.0\%. Although persistence improves further, the sensitivity and
accuracy reductions favor $m=0.05$ as the joint operating point in
this sweep.

%% file: appendices/online_diagnostics.tex
\section{Online Training Diagnostics}
\label{app:online_training_diagnostics}

These statistics are measured on each run's own online trajectories.
They complement the fixed-reference mechanism evaluation. Each method is
shown for one representative seed; faint curves show recorded values and solid curves
trailing 10-step means. Sparse GRPO+\pacg{} logs retain their recorded
steps without interpolation.

\subsection{Gate Dynamics}
\label{app:gate_dynamics}

Across the three PACG compositions, the mean span gate decreases
gradually, around half of positive direct-claim tokens are downweighted,
and a smaller fraction reaches the gate floor
(Figure~\ref{fig:gate_dynamics}). The token-multiplier mean, logged as
positive-credit retention, is not weighted by advantage magnitude.
\begin{figure}[H]
    \centering
    \includegraphics[width=\linewidth]{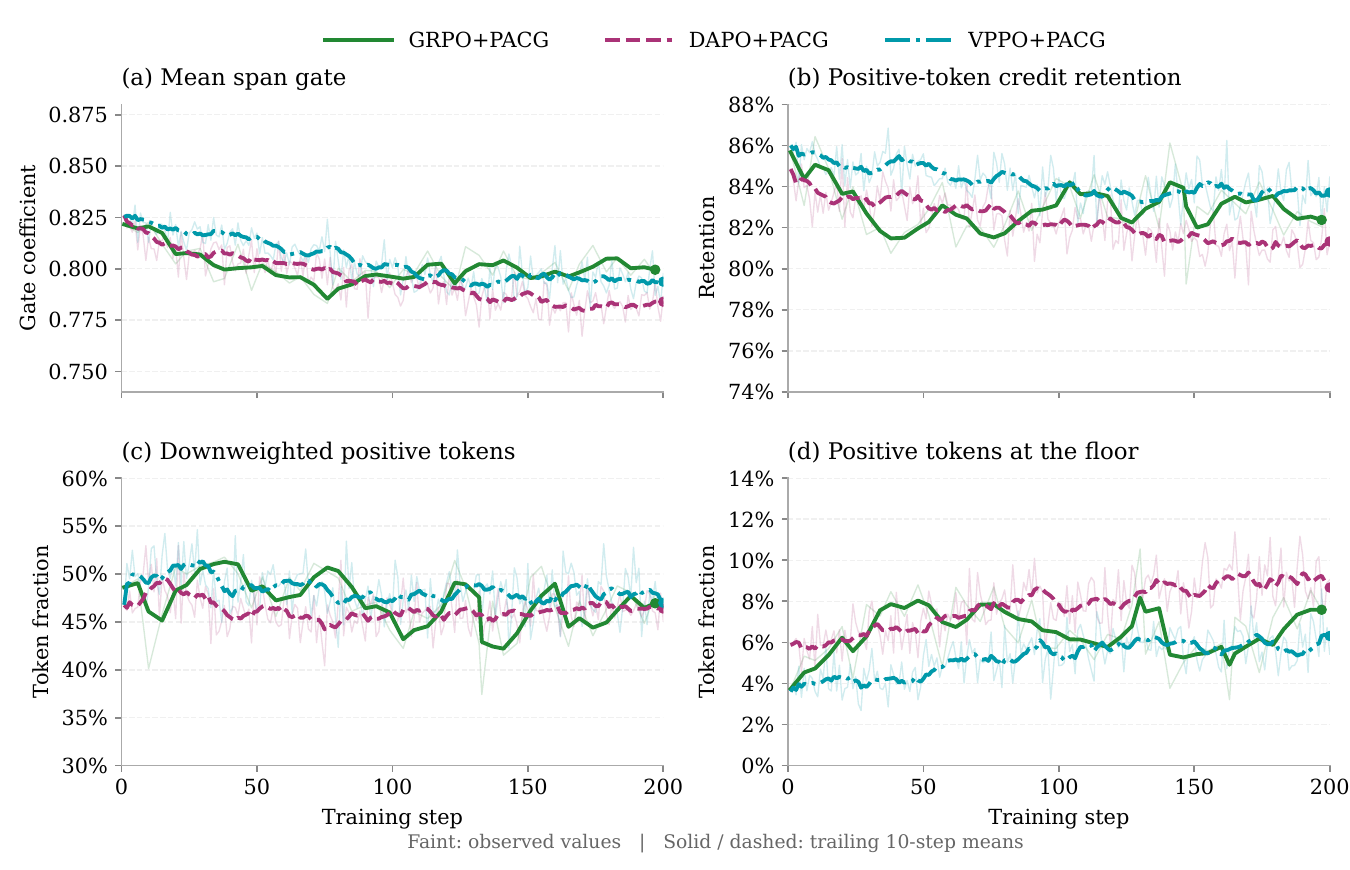}
    \caption{\textbf{PACG gate dynamics.}
(a) Mean span gate; (b) mean multiplier over positive direct-claim tokens;
(c) fraction of these tokens downweighted; (d) fraction at the gate floor.
Faint curves show observations and solid curves trailing 10-step means.}
    \label{fig:gate_dynamics}
\end{figure}

\subsection{Online Accuracy Reward}

Figure~\ref{fig:training_reward} compares each parent optimizer with its
PACG composition. PACG preserves GRPO's accuracy-reward trajectory and
raises the later trajectory for DAPO and VPPO.
\begin{figure}[H]
    \centering
    \includegraphics[width=\textwidth]
    {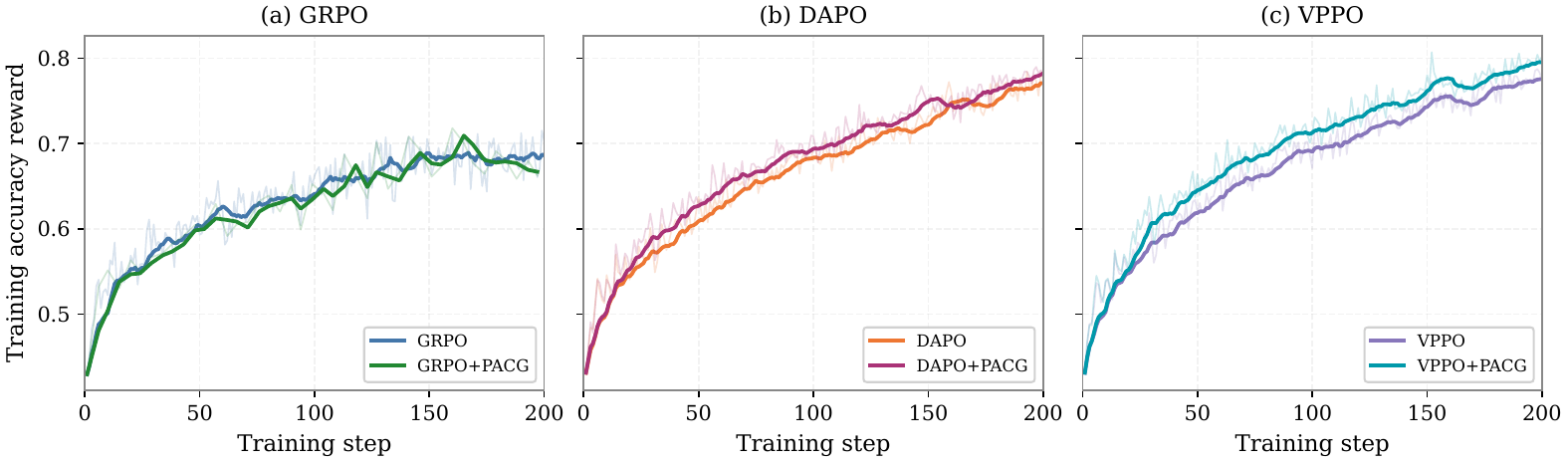}
    \caption{\textbf{Online training-reward dynamics.}
Paired parent/PACG comparisons of \texttt{reward/accuracy} through step 200.
Faint curves show logged values and solid curves trailing 10-step means.}
    \label{fig:training_reward}
\end{figure}

\subsection{VPPO Visual-Selector Scores and Gate Overlap}

VPPO+\pacg{} generally gives higher scores to selected tokens and lower
scores to unselected tokens (Figure~\ref{fig:online_vppo_visual_scores}),
showing greater separation in the online selector signal. These
policy-dependent token populations differ from the fixed direct claims
used to measure EFS and retractability.
\begin{figure}[H]
    \centering
    \includegraphics[width=\linewidth]{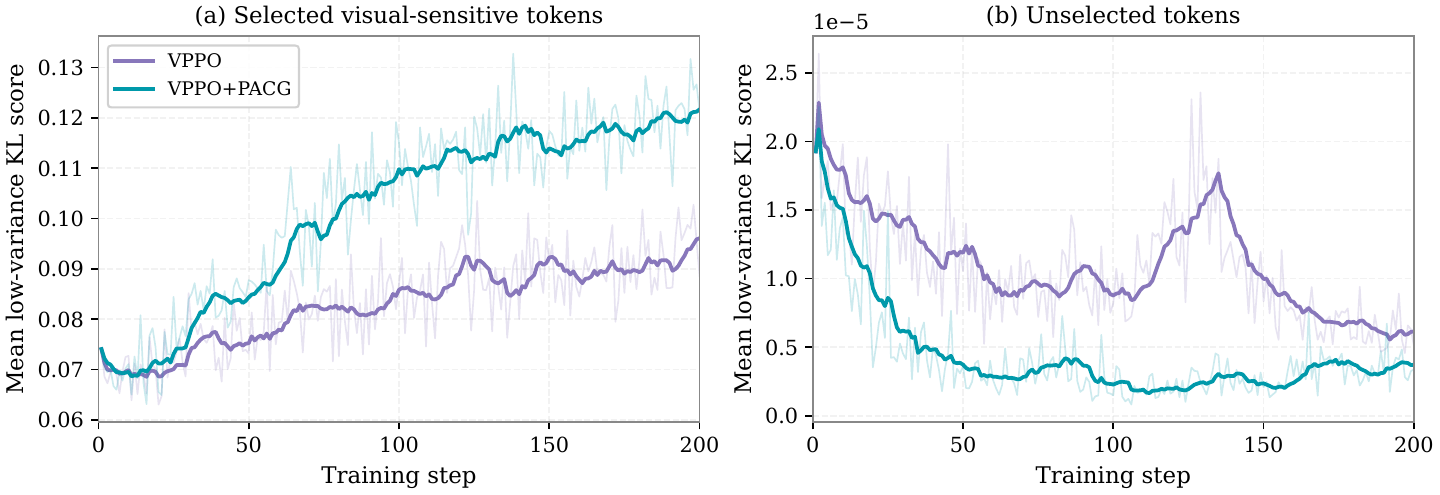}
    \caption{\textbf{Online VPPO visual-selector scores.}
Mean low-variance KL scores for (a) selected and (b) unselected tokens.
The panels use different vertical scales.}
    \label{fig:online_vppo_visual_scores}
\end{figure}

The overlap shows that visual-token selection and credit attenuation
have intersecting but distinct scopes.

\begin{figure}[H]
    \centering
    \includegraphics[width=\linewidth]{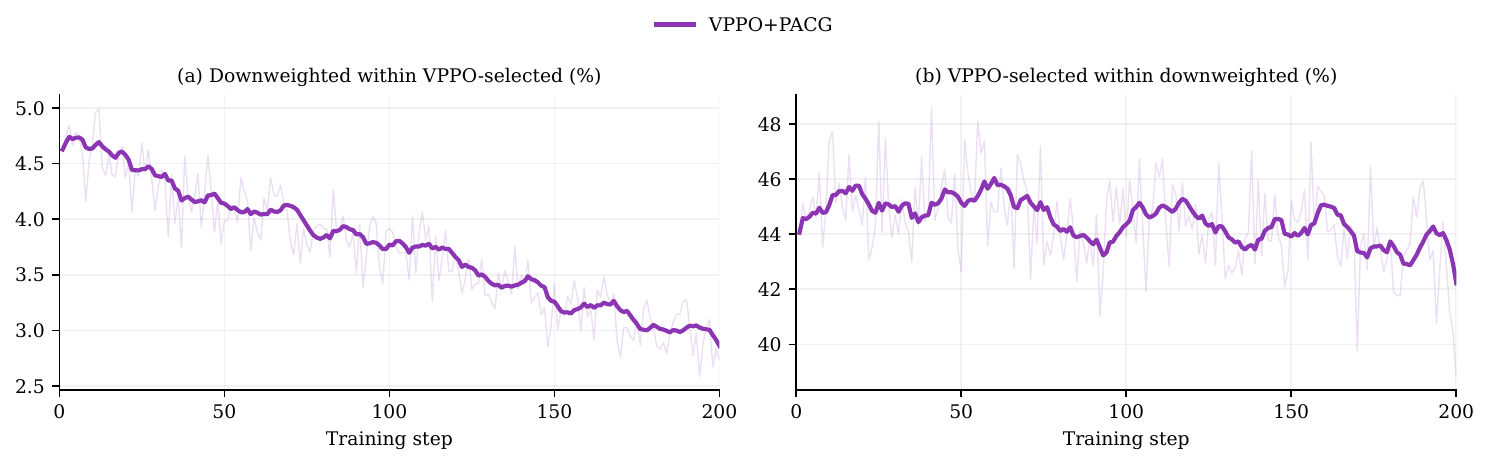}
    \caption{\textbf{Overlap between VPPO selection and PACG attenuation.}
(a) Fraction of VPPO-selected tokens downweighted by PACG;
(b) fraction of PACG-downweighted tokens selected by VPPO.
The two panels use different denominators.}
    \label{fig:vppo_pacg_overlap}
\end{figure}

\subsection{Optimization and Routing Coverage}
\label{app:additional_online_dynamics}

Figures~\ref{fig:six_run_online_dynamics} and~\ref{fig:online_persistence_coverage}
report the original exported values through step 200 without offsets or
time corrections. DAPO+\pacg{} uses $m=0.05$. Optimization statistics,
routing coverage, and selector overlap provide complementary views of
how the gate operates on the evolving response distribution.
\begin{figure}[H]
    \centering
    \includegraphics[width=0.90\linewidth]{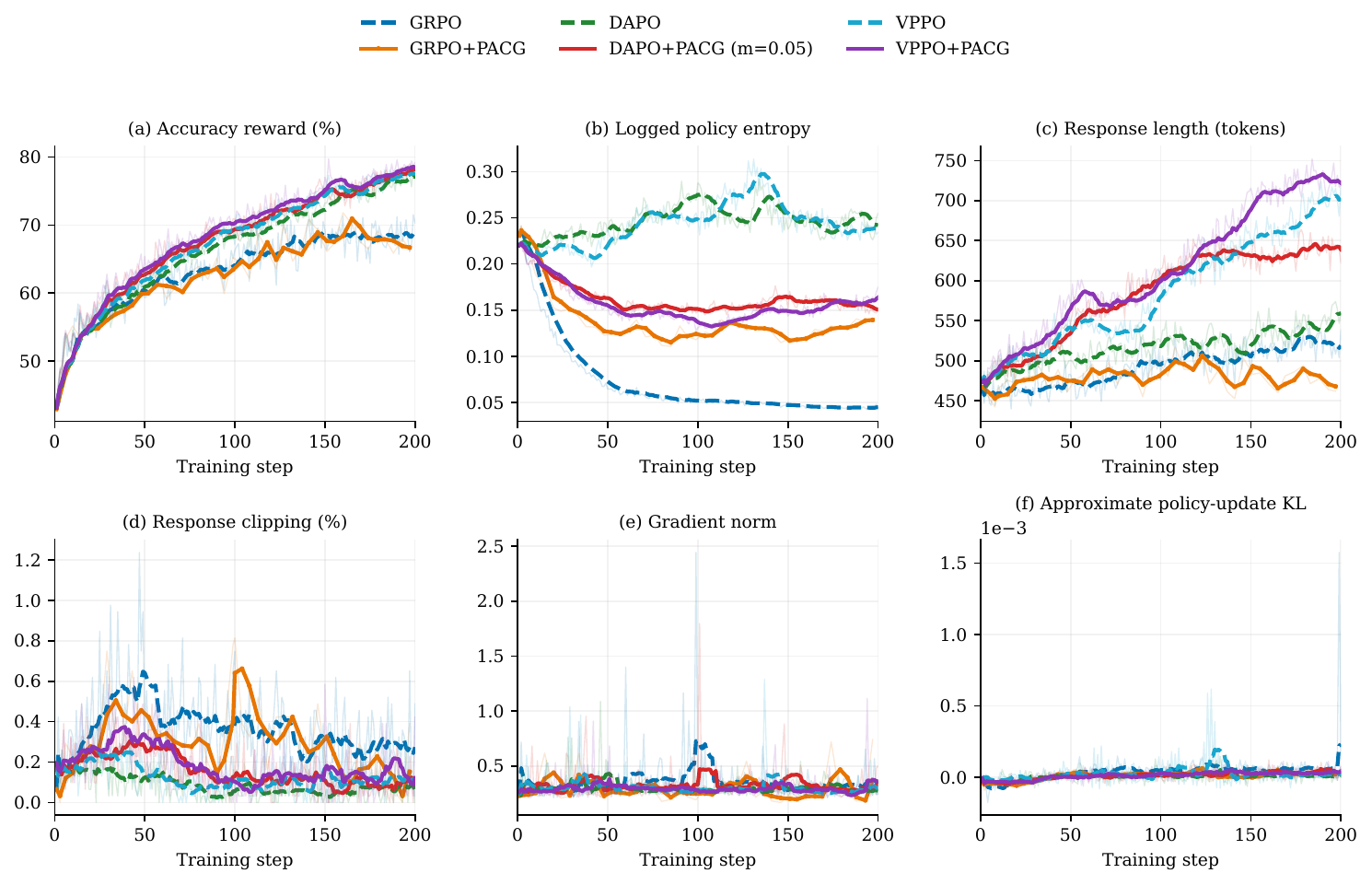}
    \caption{\textbf{Training dynamics across six runs.}
Accuracy reward, entropy, response length, clipping rate, gradient norm,
and approximate policy-update KL. Dashed curves denote parents and
solid curves their PACG compositions. Policy-update KL is distinct from
reference-model KL and image-intervention EFS.}
    \label{fig:six_run_online_dynamics}
\end{figure}

\begin{figure}[H]
    \centering
    \includegraphics[width=0.90\linewidth]{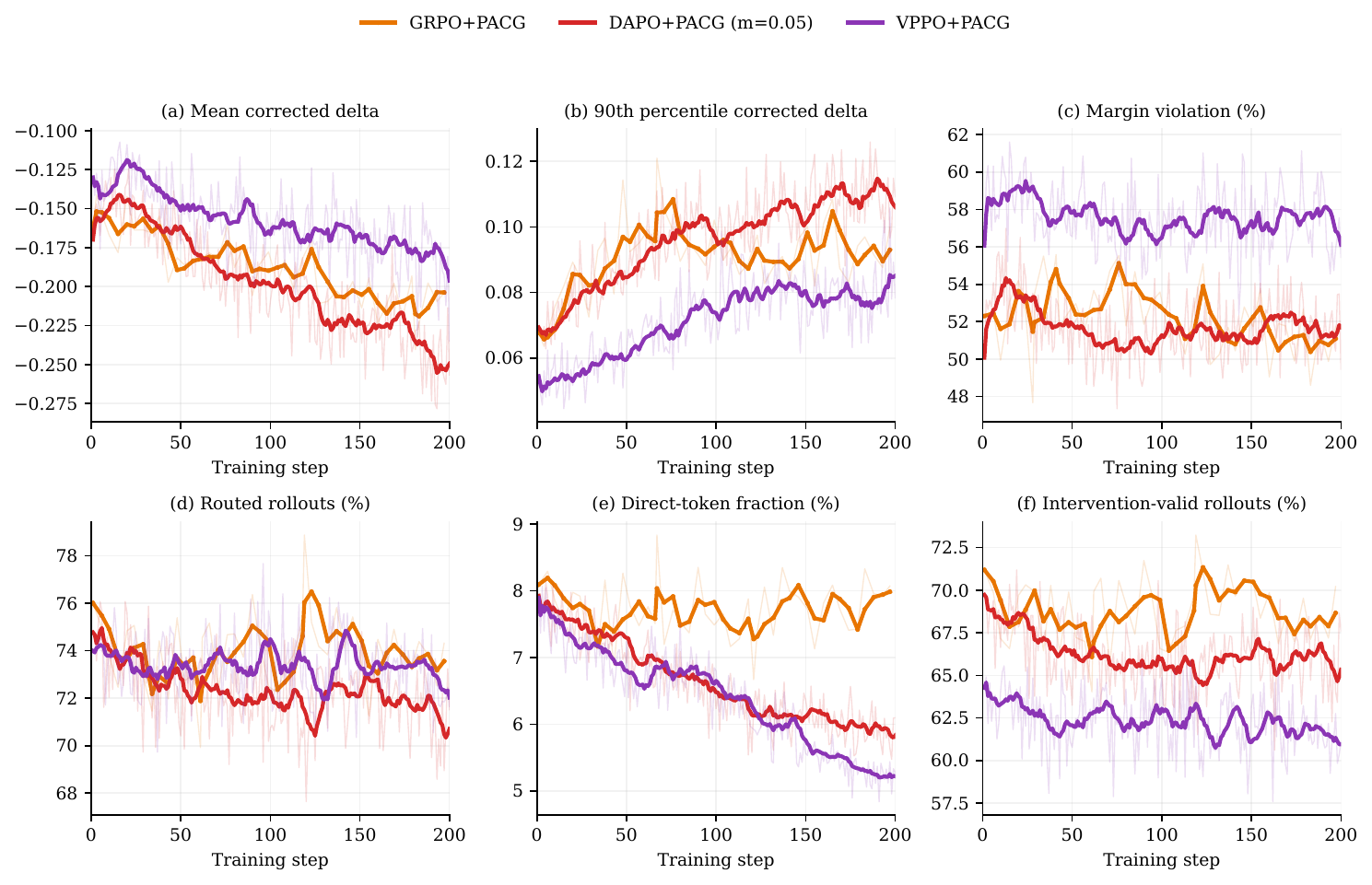}
    \caption{\textbf{Online persistence and routing coverage.}
(a)--(c) Mean corrected delta, its 90th percentile, and margin-violation
rate over eligible spans. (d)--(f) Routed-rollout rate, direct-token fraction,
and intervention-valid rate. The 90th percentile describes the span
distribution, not a confidence interval.}
    \label{fig:online_persistence_coverage}
\end{figure}

Figure~\ref{fig:online_persistence_coverage} separates average
counterfactual response, its upper tail, and the population entering the
gate. Its spans include both supported and unsupported claims.

%% file: appendices/efficiency.tex
\section{Training Efficiency}
\label{app:coverage_efficiency}

We compare DAPO and PACG ($m=0.05$) on Qwen2.5-VL-7B over training
steps 1--200 using 16 PPU-810E accelerators.
Table~\ref{tab:coverage_efficiency} gives stage-level and total times.
\begin{table}[htbp]
    \centering
    \caption{\textbf{Training efficiency on Qwen2.5-VL-7B using
16 PPU-810E accelerators.} Mean seconds per step over steps 1--200.
Total time includes validation, checkpointing, and other work.}
    \label{tab:coverage_efficiency}
    \fontsize{8.5}{10}\selectfont
    \setlength{\tabcolsep}{8pt}
    \renewcommand{\arraystretch}{1.25}
    \begin{tabular}{@{}lccc@{}}
        \toprule
        \textbf{Stage} & \textbf{DAPO} & \textbf{PACG}
        & \shortstack{\textbf{Overhead}\\\textbf{vs. DAPO}} \\
        \midrule
        Rollout & 148.8 & 174.6 & -- \\
        Original-image teacher forcing & 46.7 & 49.2 & -- \\
        Reference log-probability (KL) & 44.9 & -- & -- \\
        Augmented scoring & -- & 51.2 & -- \\
        Counterfactual teacher forcing & -- & 56.6 & -- \\
        Blocking annotation wait & -- & 159.4 & -- \\
        Actor update & 197.1 & 207.1 & -- \\
        \midrule
        \rowcolor{oursblue}
        \textbf{Total step time (s)} & \textbf{464.3} & \textbf{730.4}
        & \shortstack{$1.57\times$\\{\scriptsize(+57.3\%)}} \\
        Total logged 200-step time (h) & 25.8 & 40.6 & -- \\
        \bottomrule
    \end{tabular}
\end{table}

Mean step time increases from 464.3 to 730.4 seconds ($1.57\times$, or
57.3\% overhead), with 159.4 seconds spent waiting for annotation.
The corresponding 200-step totals are 25.8 and 40.6 hours. PACG adds
no inference-time computation.

\begin{figure}[htbp]
    \centering
    \includegraphics[width=\linewidth]{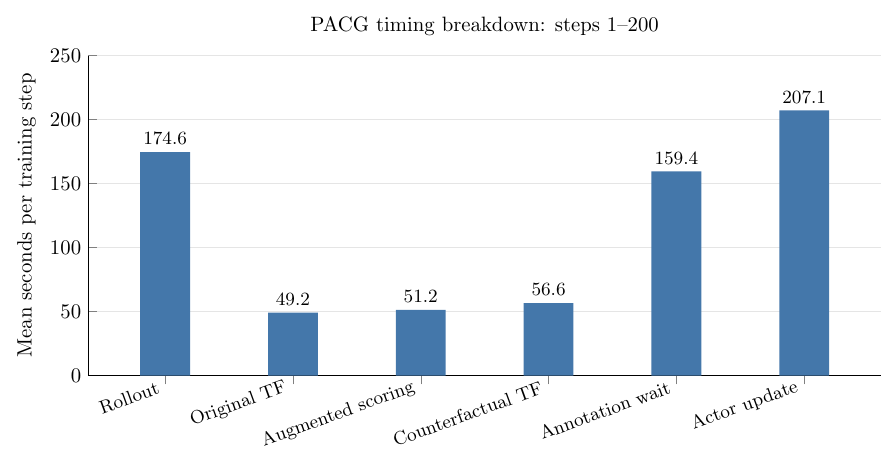}
    \caption{\textbf{Training-time breakdown for PACG ($m=0.05$).}
    Mean seconds per step over steps 1--200. TF denotes teacher forcing.
    Augmented scoring corresponds to \texttt{timing\_s/aug}.}
    \label{fig:training_timing_breakdown}
\end{figure}

%% file: appendices/prompts.tex
\section{Prompts and Templates}
\label{app:prompts}

The following prompts document span annotation, free-generation auditing,
and the standardized reasoning template used in all reported runs.

\subsection{Multimodal Span Annotation Prompt}
\label{app:router_annotation_prompt}

The following boxes reproduce the English span-annotation instructions
used in the reported runs.
The instruction and output-format portions are separated for readability;
they form one text prompt in the implementation. The request also includes
the question's image(s) as image-content items, rather than substituting
an image caption. The sample payload contains the sample identifier,
question, and trace, truncated to the configured maximum character length.
No gold answer is included. These are the exact prompt texts issued in the
reported runs.

\begin{pacgpromptbox}{Multimodal span annotation instructions}
You are a strict multimodal reasoning span annotator. You will see the
user question, its corresponding image, and a model-generated reasoning trace.

\textbf{Task.} Identify atomic reasoning spans in the trace and determine
which are direct Visual Evidence-Functional (VEF) spans. Do not judge
whether the final answer is correct. Do not conflate whether visual evidence
actually supports a span with whether that span is direct VEF.

\medskip
\textbf{Core principles}
\begin{itemize}[leftmargin=1.4em,nosep]
    \item Direct VEF is a functional label, not a correctness label.
    \item A span can be direct VEF if it claims information read directly
    from the current image, table, OCR, or spatial layout.
    \item An incorrect, hallucinated, or image-contradicting visual assertion
    can still be direct VEF.
    \item Label the relation to the image separately using the evidence
    statuses specified below.
\end{itemize}

\medskip
\textbf{A direct VEF span must satisfy all of the following}
\begin{enumerate}[leftmargin=1.4em,nosep]
    \item Its text is the smallest complete proposition occurring verbatim
    and contiguously in the trace.
    \item Its function is to read, describe, or restate visual, OCR, chart,
    counting, attribute, spatial, or topological information in the current image.
    \item The proposition cannot be obtained directly from the question
    text alone; it requires the current image or chart.
    \item It is an assertion of direct observation, without calculation,
    general rules, physical/geometric deduction, option analysis, or a final answer.
    \item It is not a hypothetical condition, metadiscourse, or merely
    a restatement of an answer option.
\end{enumerate}

\medskip
\textbf{Not direct VEF}
\begin{itemize}[leftmargin=1.4em,nosep]
    \item Facts explicitly stated in the question.
    \item Hypotheses, such as \emph{If the particle is positively charged}.
    \item General knowledge, such as \emph{the slope of a velocity-time graph
    gives acceleration}.
    \item Mathematical calculations, physical/geometric deductions,
    derived conclusions, option judgments, or final answers.
    \item Incomplete phrases, such as \emph{point a to point b} or
    \emph{from the graph}.
    \item Pure symbolic reasoning, formula substitution, or answer formatting.
    \item Sentences describing only the reasoning process or method.
    \item Meta-judgments about missing, unprovided, or indeterminable
    information, unless they directly describe visible blank areas,
    occlusion, or missing regions in the image.
\end{itemize}

\medskip
\textbf{Separate observation from inference.}
Split \emph{The velocity increases linearly, indicating a constant acceleration}
into \emph{The velocity increases linearly} (direct VEF; OCR/chart or visual
relation) and \emph{indicating a constant acceleration} (derived VEF or
symbolic deduction).

\medskip
\textbf{Missing or indeterminable information.}
Statements such as \emph{we don't have the value of y\_5 directly provided},
\emph{cannot be determined}, or \emph{not given} are usually not direct VEF.
Only consider direct VEF when the span directly describes a visible blank,
occlusion, or missing region; otherwise use derived VEF, non-VEF control,
or other.

\medskip
\textbf{Evidence status for each direct VEF span}
\begin{itemize}[leftmargin=1.4em,nosep]
    \item \texttt{supported}: the image clearly supports the proposition.
    \item \texttt{contradicted}: the image clearly negates the proposition.
    \item \texttt{unsupported}: the span claims to read information directly
    from the image, but sufficient supporting evidence cannot be found,
    without a clear contradiction.
    \item \texttt{unknown}: image evidence is insufficient to make a judgment.
\end{itemize}
A contradicted span remains direct VEF. These statuses describe evidential
support, not the VEF type.

\medskip
\textbf{Analysis categories.}
Judge visual dependency and consistency with the original image separately:
\begin{enumerate}[leftmargin=1.4em,nosep]
    \item Direct visual dependency, image-supported:
    \texttt{supported\_direct\_vef}.
    \item Direct visual dependency, contradicted or unsupported:
    \texttt{hallucinated\_or\_unsupported\_direct\_vef}.
    \item No direct visual dependency: \texttt{non\_vef\_reasoning}.
    \item Indirect visual dependency, a conclusion derived from a visual
    observation: \texttt{derived\_vef}.
    \item Direct VEF whose truth cannot be judged from the image:
    \texttt{unknown\_direct\_vef}.
\end{enumerate}
Do not omit false, fabricated, or image-contradicting visual assertions.
If their semantic function claims direct image reading, label them direct
VEF and record the lack of support through evidence status and analysis category.

\medskip
\textbf{Label vocabulary}

\texttt{span\_group}: \texttt{direct\_vef}, \texttt{derived\_vef},
\texttt{non\_vef\_control}, \texttt{other}.

\texttt{function\_type}: \texttt{visual\_relation},
\texttt{spatial\_topology}, \texttt{ocr\_chart},
\texttt{counting\_attribute}, \texttt{uncertainty},
\texttt{cross\_modal\_verification}, \texttt{derived\_vef},
\texttt{symbolic\_deduction}, \texttt{answer\_formatting}, \texttt{other}.

\texttt{assertion\_status}: \texttt{asserted}, \texttt{hypothetical},
\texttt{uncertain}, \texttt{quoted\_option}.

\medskip
\textbf{Annotation rules}
\begin{itemize}[leftmargin=1.4em,nosep]
    \item Precision first: when uncertain, do not assign direct VEF.
    \item Span text must be an exact substring of the trace. Do not
    paraphrase, translate, or modify LaTeX.
    \item Keep spans short but complete and independently interpretable.
    \item Full trace coverage is not required. Return at most 30 spans.
    \item Do not use the gold answer or consult existing Kimi/Qwen spans.
\end{itemize}

\medskip
\textbf{Checks before outputting any direct VEF span}
\begin{enumerate}[leftmargin=1.4em,nosep]
    \item Does it claim to see, read, count, or identify information in
    the current image?
    \item Does it merely repeat the question or answer options?
    \item Does it combine what is observed with what is inferred?
    \item Does it contain inference cues such as therefore, thus, hence,
    indicating, implying, so, or which means?
    \item If the latter part is removed, does the first part already
    constitute a complete visual observation?
\end{enumerate}
\end{pacgpromptbox}

\begin{pacgpromptbox}{Output format and sample payload}
Output JSON directly, not Markdown. Use the following JSON format:
\begin{verbatim}
{
  "sample_id": "...",
  "spans": [
    {
      "text": "exact substring",
      "span_group": "direct_vef",
      "function_type": "ocr_chart",
      "is_vef": true,
      "vef_role": "direct",
      "visual_dependency": "direct",
      "evidence_status": "supported",
      "analysis_category": "supported_direct_vef",
      "claim_polarity": "positive",
      "confidence": 0.95,
      "evidence_in_question_text": false,
      "requires_current_image": true,
      "contains_inference": false,
      "assertion_status": "asserted",
      "rationale": "brief reason"
    }
  ]
}
\end{verbatim}
Sample to annotate (placeholders are replaced at request time):
\begin{verbatim}
{
  "sample_id": "<sample identifier>",
  "question": "<question text>",
  "trace": "<reasoning trace>"
}
\end{verbatim}
\end{pacgpromptbox}

\subsection{Free-Generation Avoidance-Audit Prompt}
\label{app:avoidance_annotation_prompt}

For the MMK12 free-generation audit in Table~\ref{tab:avoidance_audit},
we use Qwen3.8-max as the span annotator, following the span-annotation
instructions and JSON schema shown in Prompts~1--2. Prompt~3 specifies
the message structure used by the annotation client. Each request supplies
the original image(s), question, and generated response, with an opaque
identifier that hides the source checkpoint. The client uses temperature
0, top-$p$ 1, and disables thinking. Correct and incorrect responses are
both annotated; answer scores are retained separately from the annotation
request.

\begin{pacgpromptbox}{Free-generation avoidance audit (Qwen3.8-max)}
\textbf{SYSTEM:}

You are a strict JSON annotation service. Return exactly one complete
JSON object and no reasoning, Markdown, or surrounding prose.

\medskip
\textbf{USER (multimodal content):}

\textit{[Original image(s), supplied as image-content items.]}

\textit{[The complete span-annotation instructions from Prompt~1,
followed by the JSON output schema from Prompt~2.]}

Sample to annotate:
\begin{verbatim}
{
  "sample_id": "<opaque sample identifier>",
  "question": "<question text>",
  "trace": "<freely generated response>"
}
\end{verbatim}

\medskip
\textbf{Retry suffix (only after an unsuccessful attempt):}

Retry \texttt{<attempt>}: return exactly one complete JSON object with
an explicit \texttt{"spans"} list. Every item in spans must be a JSON
object. Do not use Markdown.
\end{pacgpromptbox}

\subsection{Training and Evaluation Prompt Template}
\label{app:reasoning_template}

For all training and evaluation experiments, we use the single standardized
reasoning template below. Its structured format elicits a consistent
chain-of-thought response and supports automated parsing of final answers.
The question placeholder contains the task question and any answer choices;
the corresponding image(s) are supplied as multimodal input.

\begin{pacgpromptbox}{Reasoning template}
\textbf{SYSTEM:}

You are a helpful assistant.

\medskip
\textbf{USER:}

\texttt{\{question\}}

You first think through the reasoning process as an internal monologue,
enclosed within \texttt{<think>} \texttt{</think>} tags. Then, provide
your final answer enclosed within \texttt{\textbackslash boxed\{\}}.
\end{pacgpromptbox}

\subsection{CrossBench4 Base Annotation Request}
\label{app:crossbench_annotation_request}

The CrossBench4 Base annotation pipeline uses the same English
instructions and JSON schema shown in Prompts~1--2. It supplies each
question's image(s), question text, and generated trace to the annotation
model. It does not ask the annotator to grade the final answer or provide
the gold answer.

\begin{pacgpromptbox}{CrossBench4 Base span-annotation request}
\textbf{SYSTEM (verbatim):}

You are a strict JSON annotation service. Return exactly one complete
JSON object and no reasoning, Markdown, or surrounding prose.

\medskip
\textbf{USER:}

\texttt{<original image(s), supplied as image-content items>}

\texttt{<annotation instructions and JSON schema shown in Prompts 1--2>}

Sample to annotate:
\begin{verbatim}
{
  "sample_id": "<blind_annotation_id, or request_id if absent>",
  "question": "<question text>",
  "trace": "<trace truncated to max_trace_chars>"
}
\end{verbatim}

\medskip
\textbf{Retry suffix (after an unsuccessful attempt):}

Retry \texttt{<attempt>}: return exactly one complete JSON object with
an explicit \texttt{"spans"} list. Every item in spans must be a JSON
object. Do not use Markdown.
\end{pacgpromptbox}

%% file: appendices/limitations.tex
\section{Limitations and Broader Impact}
\label{app:limitations}

\paragraph{Evaluation scope.}
The mechanism analysis uses fixed Base responses on four multimodal
reasoning benchmarks. This controls the compared claims across checkpoints
but does not describe each policy's full generation distribution.
The separate nine-benchmark accuracy evaluation and free-generation audit
address task performance and generated claim frequency.

\paragraph{Annotation and intervention.}
Direct-claim routing and evidence-status annotation remain imperfect,
as quantified in Appendix~\ref{app:router_human_audit}.
New domains may require additional annotation calibration. Image corruption
measures a claim's response to the chosen intervention, rather than
certifying factual correctness or guaranteeing selective removal of that
claim's evidence. The robustness comparison covers three operators.

\paragraph{Calibration and compute.}
The margin sweep supports $m=0.05$ among the tested settings; it does not
jointly tune gate strength and floor across all backbones.
Online diagnostics are shown for one representative seed per method.
Routing introduces
service-dependent training latency. Appendix~\ref{app:coverage_efficiency}
reports measured stage-level timing.

\paragraph{Broader impact.}
Claim-level diagnostics can help identify unreliable visual assertions
inside apparently successful reasoning. PACG is a credit-assignment
mechanism, not a factuality verifier; deployment still requires
task-specific validation, especially in high-stakes domains.

\paragraph{Reproducibility.}
Training settings, fixed-reference construction, annotation prompts,
and additional results are documented in the preceding appendices.
Reproducing the mechanism comparisons requires the same frozen questions,
responses, spans, and image interventions.

\paragraph{Outlook.}
Reinforcement learning remains a surprisingly effective and still poorly
understood tool for multimodal reasoning. Outcome rewards alone can
substantially improve a model's reasoning ability, yet whether the same
process teaches the model to \emph{see} better, to ground each statement in
what the image actually shows, has been approached from many directions,
including perception rewards, visual-dependence reweighting, and process
supervision. Our findings add one more piece to this picture: a policy can
become more responsive to the image without becoming more willing to
withdraw what the image does not support. PACG is one attempt to encourage
the latter, so that visual claims are not only visually dependent but also
held accountable to visual evidence. We hope that the distinction between
sensitivity and retractability, and the phenomena we observe under different
optimizers, can inform future work on how multimodal models learn to look
before they reason.

%% file: appendices/qualitative_comparisons.tex
\section{Qualitative Comparisons: Parent Optimizers and PACG}
\label{app:qualitative_comparisons}

We compare selected responses from DAPO or VPPO with their
PACG counterparts on the same question and image. Each pair shows an
incorrect parent answer and a correct PACG answer; these selected samples
illustrate behavior, not its frequency. The excerpts retain the original
wording, with omissions marked by [\ldots]. Red highlights identify
erroneous visual assertions; blue highlights identify the relevant
contrasting observations. These are manual display annotations, not
router outputs. Counting examples concern object inventory rather than
the correctness of every attribute description.

The DAPO/PACG pairs precede the two VPPO/PACG pairs. For the cube-folding,
histogram, and experimental-setup excerpts, red and green highlights
retain the error/correction contrasts of the source qualitative cards.

\newtcolorbox{pacgcasebox}[1]{enhanced,
 colback=blue!2,colframe=pacgpromptblue,colbacktitle=pacgpromptblue,
 coltitle=white,fonttitle=\bfseries\small,
 fontupper=\fontsize{9.5}{12}\selectfont,
 title={#1},boxrule=0.5pt,arc=1.5pt,
 left=7pt,right=7pt,top=6pt,bottom=6pt}
\newcommand{\pacgcasebad}[1]{\textcolor{red!75!black}{\textbf{#1}}}
\newcommand{\pacgcasegood}[1]{\textcolor{blue!70!black}{\textbf{#1}}}

\subsection{Counting small, adjacent airplanes}
\label{app:paired_case_1}
\noindent
\begin{minipage}[c]{0.37\linewidth}
\centering
\includegraphics[width=\linewidth,height=1.65in,keepaspectratio]{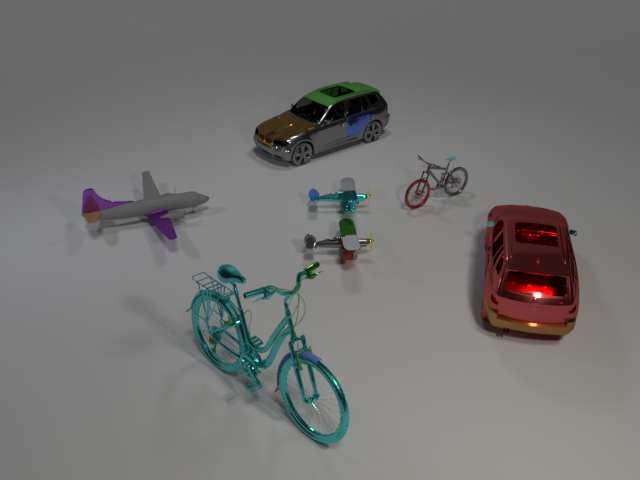}
\end{minipage}\hfill
\begin{minipage}[c]{0.60\linewidth}
\small
\textbf{Question.} How many different items are there in the image?

\medskip
\textbf{Ground truth:} 7 objects.

\end{minipage}

\medskip
\noindent
\begin{minipage}[t]{0.485\linewidth}
\vspace{0pt}
\begin{pacgcasebox}{Generated by DAPO}
\textbf{Incorrect}

\smallskip
\textcolor{black!50}{[\ldots]}

1. There is a large gray airplane on the left.\par
2. There is a small blue and white biplane in the center.\par
3. There is a large metallic bicycle in the foreground.\par
4. There is a large gray SUV in the background.\par
5. There is a large red SUV on the right.\par
6. There is a small gray bicycle near the red SUV.

\textcolor{black!50}{[\ldots]}

\pacgcasebad{There are 6 distinct items in the image.}

\(\boxed{6}\)
\end{pacgcasebox}
\end{minipage}\hfill
\begin{minipage}[t]{0.485\linewidth}
\vspace{0pt}
\begin{pacgcasebox}{Generated by DAPO+PACG}
\textbf{Correct}

\smallskip
\textcolor{black!50}{[\ldots]}

2. List the objects:\par
- There is a gray airplane on the left.\par
- There is a shiny silver car in the back.\par
- There is a small airplane with a teal and cyan body and a red tail.\par
- There is a smaller airplane with a red and silver body.\par
- There is a small bicycle with a teal frame and silver components.\par
- There is a larger bicycle with a teal frame and silver components.\par
- There is a red car on the right.

\textcolor{black!50}{[\ldots]}

3. Count the total number of unique objects:\par
- We have one gray airplane.\par
- We have two cars (one silver and one red).\par
- \pacgcasegood{We have two small airplanes.}\par
- We have two bicycles.

\textcolor{black!50}{[\ldots]}

Therefore, the total number of different items in the image is \(\boxed{7}\).
\end{pacgcasebox}
\end{minipage}

\smallskip
{\small\textbf{Reading the pair.} The image contains three airplanes, two cars, and two bicycles. DAPO lists only one of the two small central airplanes and reports six objects. The PACG inventory includes both small airplanes and reports seven.\par}

\input{appendices/qualitative_card_excerpts}

\clearpage
\subsection{Counting partially overlapping motorcycles}
\label{app:paired_case_2}
\noindent
\begin{minipage}[c]{0.37\linewidth}
\centering
\includegraphics[width=\linewidth,height=1.65in,keepaspectratio]{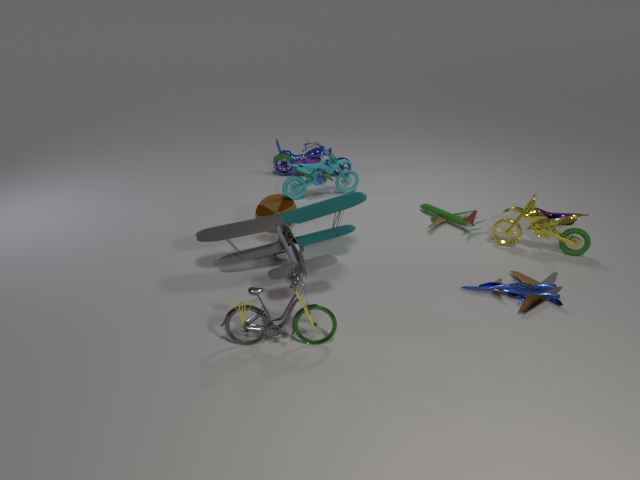}
\end{minipage}\hfill
\begin{minipage}[c]{0.60\linewidth}
\small
\textbf{Question.} How many different items are there in the image?

\medskip
\textbf{Ground truth:} 7 objects.

\end{minipage}

\medskip
\noindent
\begin{minipage}[t]{0.485\linewidth}
\vspace{0pt}
\begin{pacgcasebox}{Generated by VPPO}
\textbf{Incorrect}

\smallskip
\textcolor{black!50}{[\ldots]}

1. There is a bicycle in the foreground.\par
2. There is a biplane in the background.\par
3. There is a small airplane to the right of the biplane.\par
4. There is a blue motorcycle in the background.\par
5. There is a blue jet to the right of the small airplane.\par
6. There is a yellow motorcycle to the right of the small airplane.

\textcolor{black!50}{[\ldots]}

\pacgcasebad{Counting these, we have 6 different items.}

\(\boxed{6}\)
\end{pacgcasebox}
\end{minipage}\hfill
\begin{minipage}[t]{0.485\linewidth}
\vspace{0pt}
\begin{pacgcasebox}{Generated by VPPO+PACG}
\textbf{Correct}

\smallskip
\textcolor{black!50}{[\ldots]}

1. There is a biplane in the foreground.\par
2. There is a bicycle in the foreground.\par
3. There is a small green airplane in the background.\par
4. There is a small blue and purple motorcycle in the background.\par
5. \pacgcasegood{There is a small cyan and purple motorcycle in the background.}\par
6. There is a small yellow and purple motorcycle in the background.\par
7. There is a small blue and purple airplane in the background.

\textcolor{black!50}{[\ldots]}

There are 7 different items in the image.

\(\boxed{7}\)
\end{pacgcasebox}
\end{minipage}

\smallskip
{\small\textbf{Reading the pair.} The image contains three motorcycles, three airplanes (including the biplane), and one bicycle. VPPO lists only two motorcycles. The PACG response also identifies the cyan motorcycle and reports seven objects.\par}

\clearpage
\subsection{Reading the angle associated with the target arc}
\label{app:paired_case_3}
\noindent
\begin{minipage}[c]{0.37\linewidth}
\centering
\includegraphics[width=\linewidth,height=1.65in,keepaspectratio]{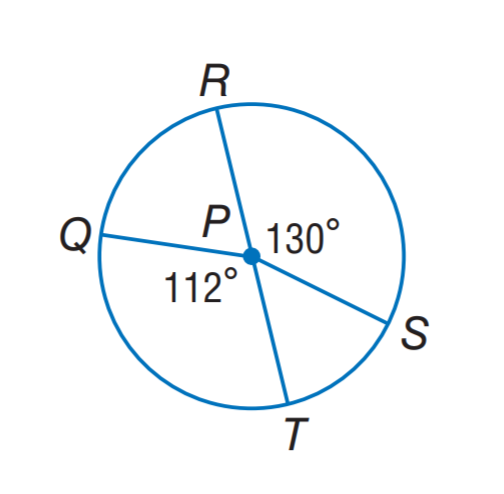}
\end{minipage}\hfill
\begin{minipage}[c]{0.60\linewidth}
\small
\textbf{Question.} In $\odot P$, the radius is $2$ inches, find the length of $\widehat {RS}$. Round to the nearest hundredth.

\medskip
\textbf{Ground truth:} \(4.54\) inches.

\end{minipage}

\medskip
\noindent
\begin{minipage}[t]{0.485\linewidth}
\vspace{0pt}
\begin{pacgcasebox}{Generated by VPPO}
\textbf{Incorrect}

\smallskip
\textcolor{black!50}{[\ldots]}

\pacgcasebad{Notice that the problem provides two central angles: \(\angle RPM = 130^\circ\) and \(\angle QPM = 112^\circ\).}

\textcolor{black!50}{[\ldots]}

The central angle subtending the arc \(\widehat{RS}\) is \(130^\circ\), and the central angle subtending the arc \(\widehat{QT}\) is \(112^\circ\).

\textcolor{black!50}{[\ldots]}

Now we use the formula for the arc length. The radius of the circle is given as 2 inches, and \pacgcasebad{the central angle corresponding to the arc \(\widehat{RS}\) is \(118^\circ\).}

\textcolor{black!50}{[\ldots]}

The answer is \(\boxed{4.12}\).
\end{pacgcasebox}
\end{minipage}\hfill
\begin{minipage}[t]{0.485\linewidth}
\vspace{0pt}
\begin{pacgcasebox}{Generated by VPPO+PACG}
\textbf{Correct}

\smallskip
\textcolor{black!50}{[\ldots]}

- The radius \(r\) of the circle is 2 inches.\par
- \pacgcasegood{The central angle \(\theta\) for the arc \(\widehat{RS}\) is \(130^\circ\).}

\textcolor{black!50}{[\ldots]}

Here, \(\theta = 130^\circ\) and \(r = 2\) inches. Plugging in these values, we get:
\[
L_{RS} = \frac{130^\circ}{360^\circ} \times 2 \pi \times 2
\]

\textcolor{black!50}{[\ldots]}

Thus, the length of the arc \(\widehat{RS}\) is \(\boxed{4.54}\) inches.
\end{pacgcasebox}
\end{minipage}

\smallskip
{\small\textbf{Reading the pair.} The diagram labels \(\angle RPS=130^\circ\); no point \(M\) is shown. VPPO invents angle labels involving \(M\), then uses \(118^\circ\) for the target arc despite initially recognizing \(130^\circ\). The selected PACG response uses \(130^\circ\) and obtains \(4.54\) inches.\par}

\FloatBarrier

%% file: appendices/qualitative_card_excerpts.tex
% Transcribed from right-hand panels of figure/qualitative_case_cards_3cases.pdf.
% These supplied excerpts have not been independently matched to raw rollouts.
\newcommand{\pacgcardgood}[1]{\textcolor{green!45!black}{\textbf{#1}}}
\newcommand{\pacgthinkopen}{\texttt{\detokenize{<think>}}}
\newcommand{\pacgthinkclose}{\texttt{\detokenize{</think>}}}

\clearpage
\subsection{Cube-folding reasoning}
\label{app:qualitative_cube_folding}
\noindent
\begin{minipage}[c]{0.40\linewidth}
\centering
\includegraphics[width=\linewidth,height=2.6in,keepaspectratio]{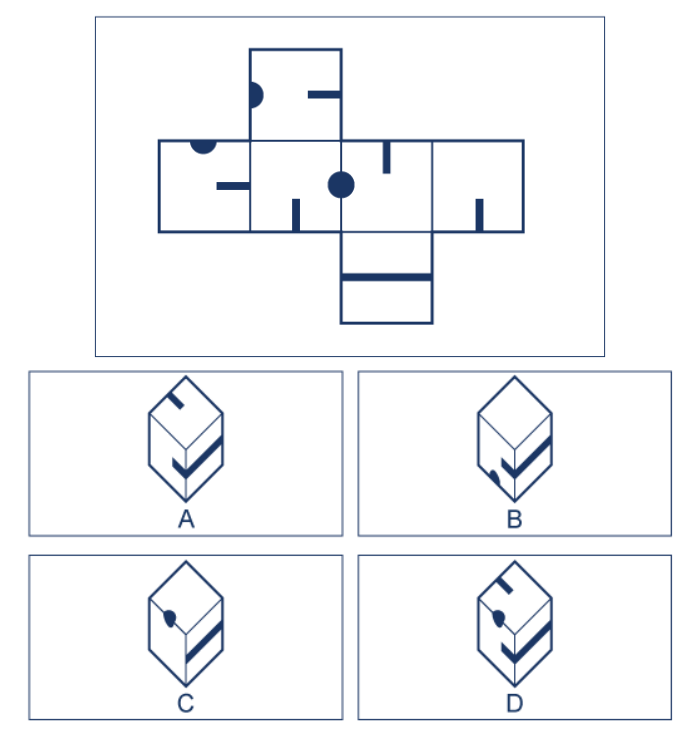}
\end{minipage}\hfill
\begin{minipage}[c]{0.57\linewidth}
\small
\textbf{Question.} Which cube can be formed by folding the given net?

\medskip
\textbf{Choices:} A, B, C, D.\quad \textbf{Ground truth:} D.

\end{minipage}

\medskip
\noindent
\begin{minipage}[t]{0.485\linewidth}
\vspace{0pt}
\begin{pacgcasebox}{Generated by DAPO}
\textbf{Incorrect}

\smallskip
\pacgthinkopen

The three faces represented in option A can be traced to neighboring faces in the net. None appears opposite another, and their line segments seem able to meet round one corner.

This \pacgcasebad{local adjacency check} makes the arrangement appear compatible with the net, so I select option A.

\pacgthinkclose

\smallskip
The predicted answer is \pacgcasebad{\(\boxed{\mathrm{A}}\)}.
\end{pacgcasebox}
\end{minipage}\hfill
\begin{minipage}[t]{0.485\linewidth}
\vspace{0pt}
\begin{pacgcasebox}{Generated by DAPO+PACG}
\textbf{Correct}

\smallskip
\pacgthinkopen

Keep the dotted central face on the left. Its right neighbor folds upward, while the face attached below folds to the right.

The dot remains on the \pacgcardgood{upper-left edge}, and the two line segments meet at the \pacgcardgood{front vertical edge}. This preserves the cyclic face order and rejects the mirrored fold.

\pacgthinkclose

\smallskip
The orientation-consistent answer is \pacgcardgood{\(\boxed{\mathrm{D}}\)}.
\end{pacgcasebox}
\end{minipage}

\clearpage
\subsection{Histogram reasoning}
\label{app:qualitative_histogram}
\noindent
\begin{minipage}[c]{0.40\linewidth}
\centering
\includegraphics[width=\linewidth,height=2.6in,keepaspectratio]{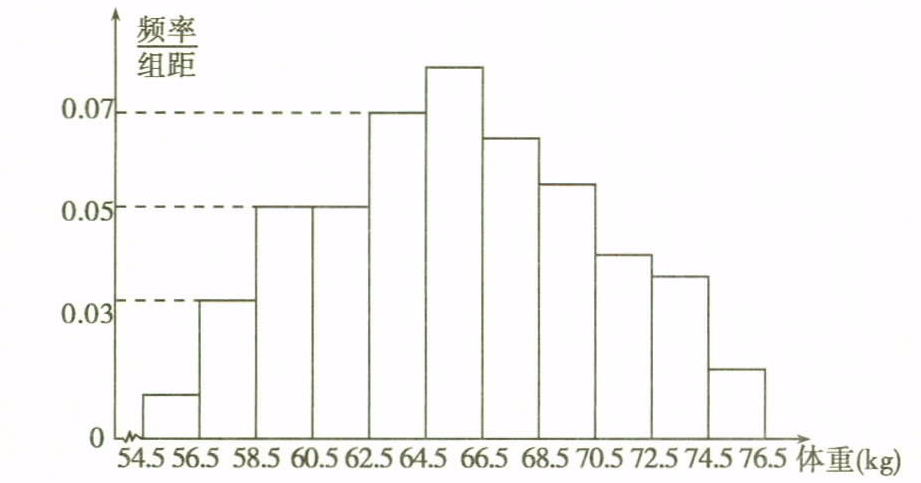}
\end{minipage}\hfill
\begin{minipage}[c]{0.57\linewidth}
\small
\textbf{Question.} Estimate how many of 1,500 boys weigh 56.5--64.5 kg.

\medskip
\textbf{Ground truth:} 600 boys.

\end{minipage}

\medskip
\noindent
\begin{minipage}[t]{0.485\linewidth}
\vspace{0pt}
\begin{pacgcasebox}{Generated by DAPO}
\textbf{Incorrect}

\smallskip
\pacgthinkopen

The interval contains four bars with heights 0.03, 0.05, 0.05, and 0.07. Adding these values gives \pacgcasebad{0.20}, which is treated as the fraction of boys in the interval.
\[
1{,}500 \times 0.20 = 300
\]
Using this height-sum interpretation, I would report 300 boys.

\pacgthinkclose

\smallskip
The predicted answer is \pacgcasebad{\(\boxed{300\text{ boys}}\)}.
\end{pacgcasebox}
\end{minipage}\hfill
\begin{minipage}[t]{0.485\linewidth}
\vspace{0pt}
\begin{pacgcasebox}{Generated by DAPO+PACG}
\textbf{Correct}

\smallskip
\pacgthinkopen

The vertical axis is \pacgcardgood{frequency per bin width}, so each bar height is a density rather than a frequency. The selected bins are each \pacgcardgood{2 kg wide}.
\[
2 \times (0.03 + 0.05 + 0.05 + 0.07) = 0.40
\]
The total bar area is therefore 0.40, and \(1{,}500 \times 0.40 = 600\).

\pacgthinkclose

\smallskip
The evidence-grounded answer is \pacgcardgood{\(\boxed{600\text{ boys}}\)}.
\end{pacgcasebox}
\end{minipage}

\clearpage
\subsection{Experimental-setup reasoning}
\label{app:qualitative_experimental_setup}
\begin{center}
\includegraphics[width=0.92\linewidth]{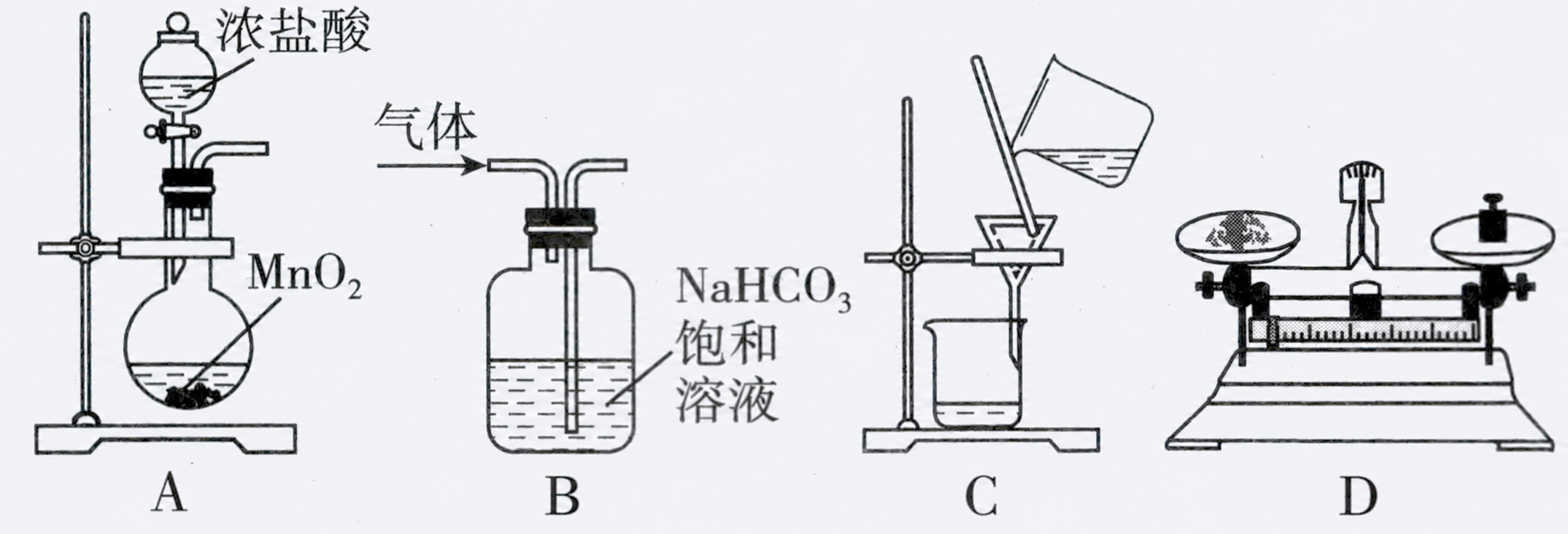}
\end{center}
{\small
\textbf{Question.} Which apparatus or experimental operation is correctly configured?

\medskip
\textbf{Choices:} A, B, C, D.\quad \textbf{Ground truth:} B.

\par}

\medskip
\noindent
\begin{minipage}[t]{0.485\linewidth}
\vspace{0pt}
\begin{pacgcasebox}{Generated by DAPO}
\textbf{Incorrect}

\smallskip
\pacgthinkopen

In setup A, manganese dioxide reacts with concentrated hydrochloric acid to produce chlorine. The reactants and gas-delivery tube are present, so the apparatus appears suitable for chlorine preparation.

This \pacgcasebad{reactant-level match} overlooks whether all required operating conditions are shown. I therefore select option A.

\pacgthinkclose

\smallskip
The predicted answer is \pacgcasebad{\(\boxed{\mathrm{A}}\)}.
\end{pacgcasebox}
\end{minipage}\hfill
\begin{minipage}[t]{0.485\linewidth}
\vspace{0pt}
\begin{pacgcasebox}{Generated by DAPO+PACG}
\textbf{Correct}

\smallskip
\pacgthinkopen

Each setup must be checked against its operating constraints. Setup A is incomplete because the \(\mathrm{MnO_2}\)--concentrated HCl reaction requires \pacgcardgood{heating}, which is absent. Setup B correctly passes the gas through \pacgcardgood{saturated \(\mathrm{NaHCO_3}\) solution} to remove an acidic impurity while retaining the target gas.

The remaining drawings contain procedural inconsistencies, so B is the only configuration that satisfies the full apparatus-level check.

\pacgthinkclose

\smallskip
The constraint-consistent answer is \pacgcardgood{\(\boxed{\mathrm{B}}\)}.
\end{pacgcasebox}
\end{minipage}
\FloatBarrier